\documentclass[5p,twocolumn]{elsarticle}
\usepackage{amsmath,amsfonts}
\usepackage{array}
\usepackage[caption=false,font=normalsize,labelfont=sf,textfont=sf]{subfig}
\usepackage{textcomp}
\usepackage{stfloats}
\usepackage{url}
\usepackage{verbatim}
\usepackage{graphicx}
\usepackage{amsmath,amssymb,amsfonts}
\usepackage[ruled]{algorithm2e} 
\SetKw{Initialize}{Initialize: }
\usepackage{algorithmic}
\usepackage{graphicx}
\usepackage{textcomp}
\usepackage{bm}

\usepackage{amsthm}

\theoremstyle{definition}
\newtheorem{definition}{Definition}
\newtheorem{proposition}{Proposition}
\newtheorem{theorem}{Theorem}
\newtheorem{lemma}{\textbf{Lemma}}
\theoremstyle{remark}

\newtheorem{corollary}{Corollary}

\newcommand{\norm}[1]{\left\lVert#1\right\rVert}
\newcommand{\abs}[1]{\lvert#1\rvert}
\DeclareMathOperator*{\argmax}{arg\,max}
\DeclareMathOperator*{\argmin}{arg\,min}

\journal{Mathematics and Computers in Simulation}

\begin{document}

	\begin{frontmatter}
	\title{A Method to Improve the Performance of Reinforcement Learning Based on the $\mathcal{Y}$ Operator for a Class of Stochastic Differential Equation-Based Child-Mother Systems}
	
	\tnotetext[t1]{This paragraph of the first footnote will contain the date on 
		which you submitted your paper for review. }
	
	\author[1,2]{Cheng Yin \corref{cor1}}
	\ead{yinchenghust@hust.edu.cn}
	
	\author[3]{Yi Chen}
	\ead{chenyi@hust.edu.cn}

	\address[1]{The School of Mechanical Science and Engineering, Huazhong University of Science and Technology, Wuhan, Hubei, China}
	\address[2]{State Key Laboratory of Intelligent Manufacturing Equipment and Technology, Huazhong University of Science and Technology, Wuhan, Hubei, China}
	\address[3]{The China-EU Institute for Clean and Renewable Energy, Huazhong University of Science and Technology, Wuhan, Hubei, China}
	
	\cortext[cor1]{Corresponding author}

	

	\begin{abstract}
		This paper introduces a novel operator, termed the $\mathcal{Y}$ operator, to elevate control performance in Actor-Critic (AC) based reinforcement learning for systems governed by stochastic differential equations (SDEs). The $\mathcal{Y}$ operator ingeniously integrates the stochasticity of a class of child-mother system into the Critic network’s loss function, yielding substantial advancements in the control performance of RL algorithms. Additionally, the $\mathcal{Y}$ operator elegantly reformulates the challenge of solving partial differential equations for the state-value function into a parallel problem for the drift and diffusion functions within the system's SDEs.  A rigorous mathematical proof confirms the operator's validity. This transformation enables the $\mathcal{Y}$ Operator-based Reinforcement Learning (YORL) framework to efficiently tackle optimal control problems in both model-based and data-driven systems. The superiority of YORL is demonstrated through linear and nonlinear numerical examples, showing its enhanced performance over existing methods post convergence.
	\end{abstract}
	
	\end{frontmatter}

	\section{INTRODUCTION}
	The advent of OpenAI's Chat Generative Pre-trained Transformer (ChatGPT) marked a seminal moment in the commercialization of Large Language Models (LLMs), signaling the advent of AI's semantic imaging prowess. This breakthrough catalyzed a paradigm shift in AI research, steering focus towards the domain of decision-making. At the forefront of this domain is reinforcement learning, which, in the wake of escalating computational capabilities, has undergone a metamorphosis into deep reinforcement learning (DRL). The seminal paper \cite{mnih2015human} ignited an explosion of interest in DRL, spawning a proliferation of sophisticated algorithms such as Deep Deterministic Policy Gradient (DDPG)\cite{lillicrap2015continuous}, Trust Region Policy Optimization (TRPO)\cite{schulman2015trust}, Asynchronous Advantage Actor Critic (A3C)\cite{mnih2016asynchronous} and Proximal Policy Optimization (PPO)\cite{schulman2017proximal}. These innovations exemplify the renaissance in reinforcement learning research.
	
	DRL distinguishes itself by utilizing deep neural networks to distill features from systemic states\cite{bohmer2015autonomous} , thus enabling a seamless bridge between action and state spaces\cite{tan2019energy} and facilitating end-to-end system optimization. This method addresses the limitations inherent in traditional control mechanisms\cite{song2023reaching} and has seen extensive exploration and implementation in robotics\cite{salvato2021crossing} and autonomous vehicular control\cite{li2022decision}.
	
	Despite the advancements in reinforcement learning (RL), its application remains constrained by the exigencies of environmental modeling during the training phase. Both model-based and model-free algorithms presuppose a level of environmental determinism that belies the inherent randomness of real-world settings. This discordance reveals that proficiency within simulated environments does not necessarily translate to real-world efficacy, as highlighted in \cite{canese2021multi}. Addressing this disjunction, researchers are exploring avenues to obviate the need for precise environmental modeling. One school of thought, as elucidated in \cite{liu2021policy}, advocates for direct agent-environment interactions, either within the actual milieu or subsequent to simulated training. However, this presents considerable safety risks accentuated by the opaque nature of the neural networks governing end-to-end deep RL, which is in stark contrast to the mathematical transparency of traditional control methods. This concern has propelled research into safety reinforcement learning, with significant contributions by \cite{thananjeyan2021recovery}, \cite{stooke2020responsive} and \cite{cheng2019end}\cite{marvi2021safe}.
	
	An alternative approach posits the integration of environmental uncertainty into simulations, with the intention of fostering more adaptable agents. Techniques such as Gaussian processes\cite{engel2005reinforcement} and Stochastic Differential Equations (SDEs)\cite{yang2023parameters} are employed to model this uncertainty. Recent innovations leverage deep neural networks to ascertain the drift and diffusion terms of SDEs from extensive environmental data\cite{xu2022infinitely}\cite{yang2023neural}. This approach called neural stochastic differential equations(NSDEs) that has gained traction in diverse fields, from finance\cite{chen2021deep} to unmanned aircraft systems\cite{djeumou_how_2023} and autonomous driving\cite{qi_stochastic_2022}.

	Furthermore, the confluence of SDE-based models and RL methodologies is yielding novel solutions to optimal control challenges\cite{wang2020reinforcement}\cite{wang2020continuous}\cite{wang1812exploration}. The \cite{zhang2022online} advanced a stochastic model for n-player non-zero sum differential games, demonstrating convergence to Nash equilibria via Q-learning. The \cite{pirmorad2021deep} employed RL to modulate a 1-dimensional Stochastic Burgers’ equation, successfully dampening shock waves and mitigating sharp gradients. Moreover, \cite{chen_incremental_2019} introduced an AC based RL framework for SDE-modeled systems, proving the convergence of the Critic network and proposing a simplified surrogate function to streamline computations.
	
	SDEs have been adeptly employed to capture the inherent stochasticity of controlled systems, presenting a distinct advantage over traditional models predicated on Ordinary Differential Equations (ODEs). This approach, showcased in \cite{zhang2022online}\cite{pirmorad2021deep}\cite{chen_incremental_2019}, offers a more nuanced representation of real-world variability. Despite these advancements, two pertinent issues emerge:
	
	\textbf{Problem I}:  The integration of stochasticity in SDE-based models does not extend to the design of the state-value function estimation in AC based reinforcement learning architectures. 
	
	Traditional deep RL algorithms design the Critic network's loss function to minimize the Mean Square Error (MSE) between the estimated value function at the current and subsequent time steps, as represented by the equations:
	\begin{equation}
		\label{loss tradition}
		\begin{aligned}
			\mathcal{L}(\theta) &= MSE\left(V^{\theta}\left(s_t\right), r_t+\gamma V^{\theta}\left(s'_t\right)\right),\\
			\mathcal{L}(\theta) &= MSE\left(Q^{\theta}\left(s_t,a_t\right), r_t+\gamma Q^{\theta}\left(s'_t,a'_t\right)\right)
		\end{aligned}
	\end{equation}
	
	In studies such as those by  \cite{zhang2022online} and \cite{pirmorad2021deep}, stochasticity is confined to the controlled system's model, with the design of RL algorithms—both Actor and Critic networks—failing to accommodate this variability. Consequently, the optimization of policy relies heavily on the generalization ability of RL methods. Moreover, \cite{chen_incremental_2019} simplifies the Critic network's loss function, derived from the It\^{o} diffusion process's characteristic operator, to a surrogate function similar to the traditional form shown in (\ref{loss tradition}) for computational efficiency, thereby overlooking the system's stochastic aspects. \cite{wang2020continuous} also primarily concentrate on the Actor network's design while employing SDEs to model the system.

	The disregard for the controlled system's stochasticity in the design of the Critic network's loss function can significantly impact its convergence, warranting a critical re-evaluation in order to align the RL algorithms with the complexity of SDE-modeled environments.

	\textbf{Problem II}: The utilization of SDEs in the modeling of control systems, as evidenced in works like \cite{chen_incremental_2019} , imposes stringent continuity conditions on the state-value function—specifically, the necessity for at least second-order continuous and $\alpha$-order H$\ddot{\text{o}}$lder continuity. This prerequisite stems from the optimization methodologies, such as the Hamilton-Jacobi-Bellman (HJB) equation or reinforcement learning algorithms, which require the state-value function's partial derivatives for optimal policy determination. This requirement restricts the effectiveness of these methods in two distinct cases:
	
	\begin{itemize}
		\item Case1: 
		The first limitation arises when the Critic network's activation function, such as the ReLU function, does not fulfill the continuity criteria. This constraint hampers the method's efficacy and limits the freedom in activation function selection, a flexibility researchers seek to retain in algorithm design.
		\item Case2: 
		In multi-objective optimization scenarios, the reward function may not be well-defined, and researchers often wish to utilize existing datasets from previous reinforcement learning agents trained on similar problems. These datasets might not satisfy the continuity prerequisites and may not even be sequentially ordered—a situation particularly relevant to inverse reinforcement learning and offline reinforcement learning.
	\end{itemize}
	
	Addressing these challenges, this paper introduces an alternative operator, referred to as the $\mathcal{Y}$ operator, that is functionally equivalent to the characteristic operator of the It\^{o} diffusion process. A rigorous proof leveraging the Kolmogorov forward equation and Gaussian distribution properties is provided in the \ref{Theory} section. Utilizing the $\mathcal{Y}$ operator within the AC framework, we design a reinforcement learning controller for the child-mother system detailed in the \ref{Model}  section.
	
	The $\mathcal{Y}$ operator not only resolves the issue of disregarding the controlled system's stochasticity in value function estimation, as stated in \textbf{Problem I}, but also circumvents the need for computing the value function's partial derivatives when deriving optimal policy, as highlighted in \textbf{Problem II}. This is achieved by $\mathcal{Y}$ operator transforming the problem of calculating the value function's partial derivatives into one of computing the drift and diffusion terms' partial derivatives in the system's describing SDEs. It provides insightful solutions for the special tasks mentioned above, including those requiring diverse activation function choices in neural network design, inverse reinforcement learning tasks, and problems agnostic to reward functions.

	The child-mother system, widely applicable in autonomous driving, is formulated as follows:
	\begin{equation}
		\begin{aligned}
			dz_t &= \mathcal{F}(z_t,u_t)dt, \\
			dw_t &= \mathcal{G}(w_t, z_t)dt \\
		\end{aligned}
	\end{equation}	
	where $z_t$ is the child-system state, $z_t$ is affected only by the child-system state $z_t$ and the control input $u_t$, and $w_t$ is the mother-system state, $w_t$ is affected only by the child-system state $z_t$ and the mother-system state $w_t$.  This type of system has a wide prospect in the field of autonomous driving, where the child-system can be considered as the ego vehicle that can be controlled and the mother-system as all surrounding vehicles excluding the ego vehicle. Most of the existing studies in the field of autonomous driving divide the child-mother system into two systems\cite{gu2022integrated}, the child-system (i.e., the ego vehicle in autonomous driving) is modeled by a two-wheeled bicycle model and a non-linear tire model\cite{snider2009automatic}, and the mother-system (i.e., the surrounding vehicles in autonomous driving) is modeled by the Intelligent Driver Model(IDE)\cite{treiber2000congested}, the Minimizing Overall Braking Induced by Lane change(MOBIL) model\cite{treiber2016mobil}, etc. However, this type of approach does not reasonably consider the stochasticity in the real model, and in the \ref{Model} section of this article, this child-mother system is modeled using data-driven NSDEs, which fully considers the uncertainty in the real environment and also gives how to calibrate the deep neural network parameters of this child-mother system using a large amount of data collected in a realistic environment.
	
	This paper proposes a novel reinforcement learning framework called $\mathcal{Y}$ Operator-based Reinforcement Learning (YORL). It contrasts this with the Traditional Stochastic Reinforcement Learning (TSRL) methodology, where the Critic network’s loss function is conventionally designed. YORL employs the $\mathcal{Y}$ operator to formulate a novel Critic network loss function for the child-mother system, as explicated in the \ref{rl design} section. The \ref{simulation} section compares the YORL and TSRL methods' performance on this this class of child-mother systems.
	
	In summary, this paper consists of the following sections. \ref{preliminaries} contains some theoretical knowledge in SDEs, which will be used in the subsequent deductions and proofs. \ref{Theory} contains the theoretical part, which proposes the $\mathcal{Y}$ operator that is equivalent to the characteristic operator of the It\^{o} diffusion process, gives a rigorous proof of the equivalence of these two operators by using the Kolmogorov forward equation and the properties of the Gaussian distribution, and finally gives two propositions of the $\mathcal{Y}$ operator. The proof of these propositions  will be shown in the Appendix. The system studied in this paper, a special class of child-mother systems, is described in \ref{Model}.  \ref{Model} details the modeling of this child-mother system using SDEs, and gives a extremely trivial method to calibrate this stochastic model using data. In \ref{rl design}, the $\mathcal{Y}$ operator proposed in \ref{Theory} is applied to the design of the loss function of the Critic network for reinforcement learning, and build a noval reinforcement learning framework based on $\mathcal{Y}$ operator(i.e., YORL). In \ref{simulation}, the reinforcement learning framework designed in \ref{rl design} is simulated in linear and nonlinear child-mother system, respectively, to solve the optimal control problem and compared with the TSRL method. \ref{conclusion} summarize the full paper and provide some outlooks for the future.

	\section{PRELIMINARIES}
	\label{preliminaries}
	For the purpose of subsequent deduction and proof, it is necessary to first give some definitions and lemmas from stochastic analysis. 
	\begin{definition} \{\cite{evans2012introduction}Section4.2: The future information\}
		
		\label{the future information}
		Let  $W(\cdot)$ be a $m$-dimensional Brownian motion defined on some probability space $(\Omega, \mathcal{F},P)$.
		The $\sigma$-algebra
		\begin{equation}
			\mathcal{W}^+(t):=\mathcal{F}\left(W\left(s\right)-W\left(t\right) \vert s\geq t\right)
		\end{equation}
		is the future information of the Brownian motion beyond time $t$.
	\end{definition}
	
	\begin{definition} \{\cite{evans2012introduction}Section4.2: Filtration\}
		
		\label{Filtration}
		In a probability space $(\Omega, \mathcal{F},P)$, a family $\{\mathcal{F}_t,t\geq 0\}$ of $\sigma$-algebras have $\forall t\geq 0, \mathcal{F}_t\subset \mathcal{F}$. $\{\mathcal{F}_t,t\geq 0\}$ is called  \textit{filtration} if there have $\forall 0\leq t\leq s, \mathcal{F}_t\subset \mathcal{F}_s$                                                                                         
	\end{definition}
	\begin{definition}\{\cite{evans2012introduction}Section4.2: Progressive process\}
		
		\label{Progressive process}
		In a probability space $(\Omega, \mathcal{F},P)$, there have a filtration $\mathcal{F}_t$ on this probability space. The mapping $G:\Omega\times[0,T] \rightarrow \mathbb{R}^n$ is a progressive process iff $G(\omega,t)$ is measurable in space $\mathcal{F}_t\times\mathcal{B}([0,T])$, $\forall t\in [0,T]$.
	\end{definition}
	\begin{definition}\{\cite{evans2012introduction}Section4.2: Two special spaces\}
		
		\label{Two special spaces}
		(i)  $\mathbb{L}^2_n(0,T)$ is the space of all real-valued, progressive processes $G: \Omega\times[0,T]\rightarrow \mathbb{R}^n$ such that
		\begin{equation}
			E\left(\int_{0}^{T}G^2dt\right)<\infty. 
		\end{equation}
		(ii) Likewise, $\mathbb{L}^1_n(0,T)$ is the space of all real-valued, progressive processes $F: \Omega\times[0,T]\rightarrow \mathbb{R}^n$ such that
		\begin{equation}
			E\left(\int_{0}^{T}\abs{F}dt\right)<\infty.
		\end{equation}
	\end{definition}
	Definition  \ref{the future information}$\sim$\ref{Two special spaces} will be used in Lemma \ref{existence and uniqueness} and are fundamentals to the subsequent description of SDEs.
	\begin{lemma}\{\cite{evans2012introduction}section5.2: Existence and uniqueness of solution of stochastic differential equation\}
		\label{existence and uniqueness}
		
		Support that $h:\mathbb{R}^n\times[0,T]\rightarrow\mathbb{R}^n$ and $H:\mathbb{R}^n\times[0,T]\rightarrow\mathbb{R}^{n\times m}$ are continuous and satisfy the following conditions:
		\begin{equation}
			(i)
			\begin{cases}
				&\forall t\in [0,T], x,\hat{x}\in \mathbb{R}^n, \exists L\in \mathbb{R} \text{ such that}\\
				\qquad &\abs{h(x,t)-h(\hat{x},t)}\leq L\abs{x-\hat{x}},\\
				&\abs{H(x,t)-H(\hat{x},t)}\leq L\abs{x-\hat{x}},\\
			\end{cases}
		\end{equation}
		\begin{equation}
			(ii)
			\begin{cases}
				&\forall t\in [0,T], x\in \mathbb{R}^n, \exists K \in \mathbb{R}\text{ such that}\\
				\qquad &\abs{h(x,t)}\leq K(1+\abs{x}),\\
				&\abs{H(x,t)}\leq K(1+\abs{x}),\\
			\end{cases}
		\end{equation}
		Let $X_0$ be any $ \mathbb{R}^n$-valued random variable such that
		\begin{equation}
			(iii)\qquad E\left(\abs{X_0}^2\right)<\infty
		\end{equation}
		and
		\begin{equation}
			(iv)
			\begin{cases}
				\qquad&X_0 \text{ is independent of } \mathcal{W}^+(0),\\
				&\mathcal{W}^+(0):=\mathcal{F}(W(s)-W(0) \vert s\geq 0).
			\end{cases}
		\end{equation}
		where $W(\cdot)$ is a given $m$-dimensional Brownian motion.
		
		Then there exists a unique solution $X \in \mathbb{L}^2_n(0,T)$ of the SDE:
		\begin{equation}
			\begin{cases}
				dX = h(X,t)dt + H(X,t)dW \qquad 0\leq t\leq T,\\
				X(0) = X_0.
			\end{cases}
		\end{equation}
	\end{lemma}
	Lemma \ref{existence and uniqueness} will be used in the \ref{Model} section of the article to calibrate the parameters of NSDEs as a gradient penalty term in the loss function.
	\section{THEORY}
	\label{Theory}
	Suppose a SDE:
	\begin{equation}
		\label{op SDE}
		dX_t = h(X_t)dt+H(X_t)dW,
	\end{equation}
	with $h(X_t)\in \mathbb{L}^1_n(0,T)$, $H(X_t)\in \mathbb{L}^2_n(0,T)$. According to lemma \ref{existence and uniqueness}, the solution of the SDE in (\ref{op SDE}) exists and is unique.
	
	\begin{lemma}\{\cite{evans2012introduction}section4.4: It\^{oz}'s chain rule\}
		
		\label{lemma ito}
		Considering a mapping $\Psi:\mathbb{R}^n\rightarrow \mathbb{R}$ is a continuous, with continuous partial derivatives $\frac{\partial\Psi(X_t)}{\partial X_{t,i}}$ and $\frac{\partial^2\Psi(X_t)}{\partial x_i\partial x_j}$ mapping, for $i, j = 1,2,\dots , n$.
		The It\^{o} chain rule is
		\begin{equation}
			\label{ito eq expand}
			\begin{aligned}
				d\Psi(X_t) = &\sum_{i=1}^{n}\frac{\partial\Psi(X_t)}{\partial X_{t,i}}(h_i(X_t)dt+(H(X_t)dW)_i)\\
				&+\frac{1}{2}\sum_{i,j=1}^{n} \frac{\partial^2\Psi(X_t)}{\partial X_{t,i}\partial X_{t,j}}(H(X_t)H(X_t)^T)_{i,j}dt,
			\end{aligned}
		\end{equation}
		where $(\cdot)_{i}$, $(\cdot)_{i,j}$ are the element of row $i$ of the matrix and the element of column $j$ of row $i$ of the matrix, respectively.
		
		Meanwhile, the expectation of (\ref{ito eq expand}) can be written as
		\begin{equation}
			\label{ito chain org}
			\begin{aligned}
				\mathbb{E}\{\frac{d\Psi(X_t)}{dt}\} &=\mathbb{E}\{  \sum_{i=1}^{n}\frac{\partial\Psi(X_t)}{\partial X_{t,i}} h_i(X_t)\\
				&+\frac{1}{2}\sum_{i,j=1}^{n}\frac{\partial^2\Psi(X_t)}{\partial X_{t,i}\partial X_{t,j}}(H(X_t)H(X_t)^T)_{i,j}  \}.
			\end{aligned}
		\end{equation}
	\end{lemma}
	
	\begin{definition}\{\cite{oksendal2003stochastic}section7.5: Characteristic Operators of It\^{o} Diffusion Process\}
		\label{characterization operator}
		
		Considering Lemma \ref{lemma ito}, introduce a characteristic operator $\mathcal{A}_X$ for the It\^{o} diffusion process.
		For a stohastic process $\{X_t\}$ in (\ref{op SDE}) and a mapping $\Psi$ in (\ref{ito eq expand}), characteristic operator $\mathcal{A}_X$ can be defined as 
		\begin{equation}
			\mathcal{A}_X\Psi(x) = \lim\limits_{t\downarrow0}\frac{\mathbb{E}\left[\Psi(X_t)\right]-\Psi(x)}{t}, \quad x\in \mathbb{R}^n
		\end{equation}
		Equivalently, characteristic operator $\mathcal{A}_X$ can be formulated as
		\begin{equation}
			\begin{aligned}
				\mathcal{A}_X(\cdot) &= \sum_{i=1}^{n}\frac{\partial(\cdot)}{\partial X_{t,i}} h_i(X_t)\\
				&+\frac{1}{2}\sum_{i,j=1}^{n}\frac{\partial^2(\cdot)}{\partial X_{t,i}\partial X_{t,j}}\left(H(X_t)H(X_t)^T\right)_{i,j}.
			\end{aligned}
		\end{equation}
		Likewise, let $\mathcal{A}^*_X$ be the dual of $\mathcal{A}_X$
		\begin{equation}
			\begin{aligned}
				\mathcal{A}^*_X(\cdot) &= -\sum_{i=1}^{n}\frac{\partial\left[\left(\cdot\right)h_i\left(X_t\right)\right]}{\partial X_{t,i}}\\
				&+\frac{1}{2}\sum_{i,j=1}^{n}\frac{\partial^2\left[\left(\cdot\right)\left(H\left(X_t\right)H\left(X_t\right)^T\right)_{i,j}\right]}{\partial X_{t,i}\partial X_{t,j}}.
			\end{aligned}
		\end{equation}
	\end{definition}
	\begin{corollary}\{\cite{oksendal2003stochastic}section7.3: The generator of It\^{o} Diffusion\}
		\label{characterization operator1}
		
		Assuming a random variable $B$, the probability density function of it is denoted as $p(b)$. Denoted
		\begin{equation}
			<B, p(b)> = \mathbb{E}\{B\} = \int_{\mathbb{R}}Bp(b)db
		\end{equation}
		as the expectation of the random variable $B$.
		It's trivial that the (\ref{ito chain org}) can be reformulated as
		\begin{equation}
			\mathbb{E}\{\frac{d\Psi(X_t)}{dt}\}  = <\mathcal{A}_X\Psi(X_t), p(X_t)>
		\end{equation}
		where $p(X_t)$ is the probability density function of random variable $X_t$ in (\ref{op SDE}).
	\end{corollary}
	
	This corollary is the cornerstone of the subsequent definition of the $\mathcal{Y}$ operator and the design of the loss function for the Critic network in Section \ref{rl design}. It transforms the problem of solving for the derivative of the function $\Psi(X_t)$ with respect to time $t$ into the problem of solving for the partial derivative of the function $\Psi(X_t)$ with respect to its independent variables. This is a very meaningful transformation.
	\begin{corollary}\{\cite{ludvigsson2013kolmogorov}section3.2: Kolmogorov Forward Equation\}
		\label{characterization operator2}
		
		According to Kolmogorov forward equation\cite{ludvigsson2013kolmogorov},  the characteristic operator $\mathcal{A}_X$ and the dual form $\mathcal{A}^*_X$ satisfy the following relationship
		\begin{equation}
			<\mathcal{A}_X\Psi(X_t), p(X_t)> = <\Psi(X_t), \mathcal{A}^*_Xp(X_t)>
		\end{equation}
	\end{corollary}
	
	In some special scenarios, the function $\Psi(X_t)$ is not agnostic and the only information available is the input and output data of the function $\Psi(X_t)$ ($\Psi(X_t)$ is known to be a continuous function or $\Psi(X_t)$ can be constructed as a continuous function). Therefore, in this type of scenario, the partial derivation of function $\Psi(X_t)$ is not known, is there any other way to solve the problem in this type of scenario?
	A new operator proposed in this paper is introduced and the relationship between the $\mathcal{Y}$ operator and the characteristic operator of It\^{o} diffusion process $\mathcal{A}$ is given.
	\begin{definition}($\mathcal{Y}$ operator)
		\label{Y operator}
		
		The $\mathcal{Y}$ operator with respect to the stochastic process $\{X_t\}$ can be formulated as
		\begin{equation}
			\begin{aligned}
				\mathcal{Y}_X\Psi&(X_t) =\\
				&\Psi(X_t)\{( (\Sigma^{X}_t)^{-1}(X_t - E_{\mu,X}^t))^T h(X_t)\\
				&-\varLambda^T \mathcal{C}^1_{h,x,t}  +\frac{1}{2} \varLambda^T \{ [-(\Sigma^{X}_t)^{-1}\\
				&+(\Sigma^{X}_t)^{-1} (X_t - E_{\mu,X}^t) (X_t - E_{\mu,X}^t) ^T ( (\Sigma^{X}_t)^{-1}) ^T ]\\
				&\odot (H(X_t) H(X_t)^T)+ \mathcal{C}^2_{H,x,t} \}\varLambda \\
				&- \varLambda^T \mathcal{C}^1_{H,x,t} (\Sigma^{X}_t)^{-1}(X_t - E_{\mu,X}^t) \}\\			
			\end{aligned}
		\end{equation}
		where $\mathcal{C}^1_{h,x,t}$, $ \mathcal{C}^1_{H,x,t}$ and $\mathcal{C}^2_{H,x,t}$ are the first-order partial differential matrix with respect to the function $h(X_t)$, the first-order partial differential matrix with respect to the function $H(X_t)$ and the second-order partial differential matrix with respect to the function $H(X_t)$, respectively.
		\begin{equation}
			\begin{aligned}
				\mathcal{C}^1_{h,x,t} &= 
				\begin{bmatrix}
					\frac{\partial h_1(X_t)}{\partial X_{t,1}}\\
					\vdots\\
					\frac{\partial h_n(X_t)}{\partial X_{t,n}}
				\end{bmatrix}\\
				\mathcal{C}^1_{H,x,t} &= \\
				&\begin{bmatrix}
					\frac{\partial (H(X_t) H(X_t) ^T)_{1,1}}{\partial X_{t,1}}  & \cdots & \frac{\partial (H(X_t) H(X_t) ^T)_{1,n}}{\partial X_{t,1}}\\
					\frac{\partial (H(X_t) H(X_t) ^T)_{2,1}}{\partial X_{t,2}}  &\cdots &\frac{\partial (H(X_t) H(X_t) ^T)_{2,n}}{\partial X_{t,2}}\\
					\vdots  & \cdots & \vdots \\
					\frac{\partial (H(X_t) H(X_t) ^T)_{n,1}}{\partial X_{t,n}}  &\cdots & \frac{\partial (H(X_t) H(X_t) ^T)_{n,n}}{\partial X_{t,n}}
				\end{bmatrix}\\
				\mathcal{C}^2_{H,x,t} &= \\
				&\begin{bmatrix}
					\frac{\partial^2 (H(X_t) H(X_t) ^T)_{1,1}}{\partial X^2_{t,1}}  & \cdots & \frac{\partial^2 (H(X_t) H(X_t) ^T)_{1,n}}{\partial X_{t,1} \partial X_{t,n}}\\
					\frac{\partial^2 (H(X_t) H(X_t) ^T)_{2,1}}{\partial X_{t,2} \partial X_{t,1}}  &\cdots &\frac{\partial^2 (H(X_t) H(X_t) ^T)_{2,n}}{\partial X_{t,2} \partial X_{t,n}}\\
					\vdots  & \cdots & \vdots \\
					\frac{\partial^2 (H(X_t) H(X_t) ^T)_{n,1}}{\partial X_{t,n} \partial X_{t,1}}  &\cdots & \frac{\partial^2 (H(X_t) H(X_t) ^T)_{n,n}}{\partial X^2_{t,n} }
				\end{bmatrix}\\
			\end{aligned}
		\end{equation}
		The $E_{\mu,X}^t$ and $\Sigma^{X}_t$ are the expectation and covariance matrices of the stochastic process $\{X_t\}$ at time $t$, respectively.
		\begin{equation}
			\begin{aligned}
				E_{\mu,X}^t = X_{t-}+h(X_{t-})dt,\\
				\Sigma^{X}_t = H(X_{t-}) H(X_{t-}) ^Tdt.
			\end{aligned}
		\end{equation}
		where $X_{t-}$ is the state of the stochastic process $\{X_t\}$ at an extremely short moment before time $t$, i.e., the history information of the last extremely short moment of $X_t$.
		
		Moreover, $\varLambda$ is the all-one vector and the dimension of it is equal to the dimension of $X_t$ which is $n$. The operator $\odot$ is the Hadamard product.
	\end{definition}
	\begin{theorem}(Operator Equivalence Theorem)
		\label{operator equal theorem}
		
		When the $\mathcal{Y}$ operator is given well-defined as described in Definition \ref{Y operator}, the $\mathcal{Y}$ operator with respect to the stochastic process $\{X_t\}$, which is $\mathcal{Y}_X$, is equivalent to the characteristic operator of It\^{o} diffusion process $\mathcal{A}_X$.
		\begin{equation}
			<\mathcal{A}_X\Psi(X_t), p(X_t)> = <\mathcal{Y}_X\Psi(X_t), p(X_t)>
		\end{equation}
	\end{theorem}
	\begin{proof}
		According to Corollary \ref{characterization operator2}, $<\mathcal{A}_X\Psi(X_t), p(X_t)>$ can be formulated as
		\begin{equation}
			\label{char op x}
			\begin{aligned}
				<&\mathcal{A}_X\Psi(X_t), p(X_t)> =  <\Psi(X_t), \mathcal{A}^*_Xp(X_t)>\\
				&=\int_{\Omega} \Psi(X_t)\{-\sum_{i=1}^{n}\frac{\partial p(X_t)}{\partial X_{t,i}}h_i(X_t)+\frac{\partial h_i(X_t)}{\partial X_{t,i}}p(X_t)\\
				&+\frac{1}{2}\sum_{i,j}^{n} \frac{\partial^2p(X_t)}{\partial X_{t,i}\partial X_{t,j}} (H(X_t) H(X_t) ^T)_{i,j}\\
				&+  \frac{\partial^2 (H(X_t) H(X_t) ^T)_{i,j}}{\partial X_{t,i}\partial X_{t,j}} p(X_t)\\
				&+ \frac{\partial p(X_t)}{\partial X_{t,i}} \frac{\partial (H(X_t) H(X_t) ^T)_{i,j}}{\partial X_{t,j}} \\
				&+  \frac{\partial p(X_t)}{\partial X_{t,j}} \frac{\partial (H(X_t) H(X_t) ^T)_{i,j}}{\partial X_{t,i}} \}dX_t.
			\end{aligned}
		\end{equation}
		Define the following notations, 
		\begin{equation}
			\begin{aligned}
				\frac{\partial p(X_t)}{\partial X_t} &= \begin{bmatrix}
					\frac{\partial p(X_t)}{\partial X_{t,1}}\\
					\frac{\partial p(X_t)}{\partial X_{t,2}}\\
					\vdots\\
					\frac{\partial p(X_t)}{\partial X_{t,n}}
				\end{bmatrix},\\
				\frac{\partial^2 p(X_t)}{\partial X^2_{t}}  &= \\
				&\begin{bmatrix}
					\frac{\partial^2 p(X_t)}{\partial X^2_{t,1}}& \frac{\partial^2 p(X_t)}{\partial X_{t,1}\partial X_{t,2}}  & \cdots & \frac{\partial^2 p(X_t)}{\partial X_{t,1} \partial X_{t,n}}\\
					\frac{\partial^2 p(X_t)}{\partial X_{t,2} \partial X_{t,1}}& \frac{\partial^2 p(X_t)}{\partial X^2_{t,2} }  &\cdots &\frac{\partial^2 p(X_t)}{\partial X_{t,2} \partial X_{t,n}} \\
					\vdots&\vdots  & \ddots & \vdots \\
					\frac{\partial^2 p(X_t)}{\partial X_{t,n} \partial X_{t,1}}&\cdots  &\cdots & \frac{\partial^2 p(X_t)}{\partial X^2_{t,n} }
				\end{bmatrix},\\
			\end{aligned}
		\end{equation}
		the (\ref{char op x}) can be reformulated into the matrix form
		\begin{equation}
			\begin{aligned}
				<\mathcal{A}_X&\Psi(X_t), p(X_t)>=\\
				&\int_{\Omega} \Psi(X_t)\{-(\frac{\partial p(X_t)}{\partial X_t})^T h(X_t)- p(X_t) \varLambda^T\mathcal{C}^1_{h,x,t}\\
				&+\frac{1}{2} [\varLambda^T \frac{\partial^2 p(X_t)}{\partial X^2_{t}} \odot (H(X_t) H(X_t) ^T) \varLambda\\
				&+p(X_t)  \varLambda^T  \mathcal{C}^2_{H,x,t}\varLambda + 2\varLambda^T \mathcal{C}^1_{H,x,t} \frac{\partial p(X_t)}{\partial X_{t}} ]\}dX_t
			\end{aligned}
		\end{equation}
		From the description of the It\^{o} diffusion process defined in (\ref{op SDE}), the random variable $X_t$ which is the stochastic process $\{X_t\}$ at moment $t$ obeys a Gaussian distribution for  as follows
		\begin{equation}
			X_{t} \sim \mathcal{N}(X_{t-}+h(X_{t-})dt ,H(X_{t-})H(X_{t-})^Tdt),\\
		\end{equation}
		Therefore, the probabilty density function of random variable $X_t$ is 
		\begin{equation}
			\label{prob x}
			\begin{aligned}
				&p(X_t) = \frac{e^{I^{X}_t}}{(2\pi)^{\frac{n}{2}}\sqrt{dt}\abs{\det(H(X_{t-}))}},\\
				&I^{X}_t = {-\frac{1}{2}(X_{t}-E_{\mu,X}^t)^T(\Sigma^{X}_t)^{-1}(X_{t}-E_{\mu,X}^t)},\\
				&E_{\mu,X}^t = X_{t-}+h(X_{t-})dt,\\
				&\Sigma^{X}_t = H(X_{t-})H(X_{t-})^Tdt,\\
			\end{aligned}
		\end{equation}
		
		It's trivial that the following formulation holds
		\begin{equation}
			\label{prob x partial}
			\begin{aligned}
				\frac{\partial p(X_t)}{\partial X_{t}} &= p(X_t)\frac{\partial \ln p(X_t)}{\partial X_{t}},\\
				\frac{\partial^2 p(X_t)}{\partial X_{t}^2} &= p(X_t)\frac{\partial^2 \ln p(X_t)}{\partial X_{t}^2}+\frac{1}{p(X_t)}\frac{\partial p(X_t)}{\partial X_t}(\frac{\partial p(X_t)}{\partial X_t})^T.\\
			\end{aligned}
		\end{equation}
		By the property that $\Sigma^{X}_t$ is symmetric matrix and (\ref{prob x}), $\frac{\partial \ln p(X_t)}{\partial X_{t}}$ and $ \frac{\partial^2 \ln p(X_t)}{\partial X_{t}^2}$ are clearly
		\begin{equation}
			\begin{aligned}
				\frac{\partial \ln p(X_t)}{\partial X_{t}} &= -(\Sigma^{X}_t)^{-1}(X_t - E_{\mu,X}^t),\\
				\frac{\partial^2 \ln p(X_t)}{\partial X_{t}^2} &= -(\Sigma^{X}_t)^{-1}.
			\end{aligned}
		\end{equation}
		Then, the (\ref{prob x partial}) can be rewritten as
		\begin{equation}
			\label{prob x partial new}
			\begin{aligned}
				\frac{\partial p(X_t)}{\partial X_{t}} &= -p(X_t)(\Sigma^{X}_t)^{-1}(X_t - E_{\mu,X}^t),\\
				\frac{\partial^2 p(X_t)}{\partial X_{t}^2} &= p(X_t)(\Sigma^{X}_t)^{-1}[-\mathit{I}_E\\
				&+(X_t - E_{\mu,X}^t)(X_t - E_{\mu,X}^t)^T((\Sigma^{X}_t)^{-1})^T],\\
			\end{aligned}
		\end{equation}
		where $I_E$ is the identity matrix. 
		
		Therefore, $<\mathcal{A}_X\Psi(X_t), p(X_t)>$ can be rewritten as
		\begin{equation}
			\label{Y operator final}
			\begin{aligned}
				<\mathcal{A}_X&\Psi(X_t), p(X_t)> =\\
				&\int_{\Omega} \Psi(X_t) \{((\Sigma^{X}_t)^{-1}(X_t - E_{\mu,X}^t))^T h(X_t)\\
				&-  \varLambda^T\mathcal{C}^1_{h,x,t} +\frac{1}{2} \varLambda^T \{  [-(\Sigma^{X}_t)^{-1}\\
				&+(\Sigma^{X}_t)^{-1}(X_t - E_{\mu,X}^t)(X_t - E_{\mu,X}^t)^T((\Sigma^{X}_t)^{-1})^T]\\
				&\odot (H(X_t) H(X_t)^T) + \mathcal{C}^2_{H,x,t}\}\varLambda \\
				&- \varLambda^T \mathcal{C}^1_{h,x,t} (\Sigma^{X}_t)^{-1}(X_t - E_{\mu,X}^t) \}p(X_t)dX_t
			\end{aligned}
		\end{equation}
		It is obvious that the holding of (\ref{Y operator final}) is a proof of the holding of Theorem \ref{operator equal theorem}.
	\end{proof}
	\begin{proposition}
		\label{theory proposition1}
		Assuming $X^{(1)}_{t}$, $X^{(2)}_{t}$ $\dots$ $X^{(m)}_{t}$ are obey It\^{o} diffusion process as (\ref{op SDE})
		\begin{equation}
			\begin{cases}
				dX^{(1)}_{t} = h^{(1)}(X^{(1)}_{t} \dots X^{(m)}_{t})dt\\
				\qquad\qquad +H^{(1)}(X^{(1)}_{t} \dots X^{(m)}_{t})dW^{(1)}_t\\
				dX^{(2)}_{t} = h^{(2)}(X^{(1)}_{t} \dots X^{(m)}_{t})dt \\
				\qquad\qquad+H^{(2)}(X^{(1)}_{t} \dots X^{(m)}_{t})dW^{(2)}_t\\
				\qquad\qquad\qquad\vdots \\
				dX^{(m)}_{t} = h^{(m)}(X^{(1)}_{t} \dots X^{(m)}_{t})dt \\
				\qquad\qquad+H^{(m)}(X^{(m)}_{t} \dots X^{(m)}_{t})dW^{(m)}_t\\
			\end{cases}
		\end{equation}
		If $dW^{(1)}_{t}$, $dW^{(2)}_{t}$ $\dots$ $dW^{(m)}_{t}$ are independent of each other and there exists a function $\Psi(X^{(1)}_{t},X^{(2)}_{t}\dots X^{(m)}_{t})$ of $X^{(1)}_{t}$, $X^{(2)}_{t}$ $\dots$ $X^{(m)}_{t}$.
		Then the expectation of the derivative with respect to time of a multivariate continuous function $\Psi(X^{(1)}_{t},X^{(2)}_{t}\dots X^{(m)}_{t})$ satisfying the conditions in Lemma \ref{lemma ito} is of the form
		\begin{equation}
			\label{theory proposition1 eq}
			\begin{aligned}
				\mathbb{E}\{&\frac{d \Psi(X^{(1)}_{t},X^{(2)}_{t}\dots X^{(m)}_{t})}{dt}\} = <\mathcal{Y}_{X^{(1)}}\Psi, p(X^{(1)}_{t})>\\
				&+<\mathcal{Y}_{X^{(2)}}\Psi, p(X^{(2)}_{t})>+\cdots +<\mathcal{Y}_{X^{(n)}}\Psi, p(X^{(m)}_{t})>
			\end{aligned}
		\end{equation}
	\end{proposition}
	\begin{proposition}
		\label{theory proposition3}
		$\mathcal{Y}$ operator is a linear operator. Therefore, it is not necessary to prove that the following equation clearly holds
		\begin{equation}
			<\mathcal{Y}_X\Psi(X_t), p(X_t)> = <\Psi(X_t), \mathcal{Y}_Xp(X_t)>
		\end{equation}
	\end{proposition}
	\begin{proposition}
		\label{theory proposition2}
		The $\mathcal{Y}$ operator transforms the problem of solving the partial derivatives of the function $\Psi(X_t)$ with respect to the independent variables($X_t$) into solving the partial derivatives of the drift term function and the diffusion term function of the SDE with respect to their independent variables($X_t$).
		\begin{equation}
			\label{theory proposition2 eq1}
			 \mathcal{A}_{X}\Psi(X_t) = \mathcal{Y}_X\Psi(X_t)
		\end{equation}
		Moreover, according to Kolmogorov forward equation, the following equation holds
		\begin{equation}
			\label{theory proposition2 eq2}
			\frac{d p(X_t)}{d t} = \mathcal{A}^*_{X}p(X_t) = \mathcal{Y}_Xp(X_t)
		\end{equation}
	\end{proposition}
	
	Proofs of Propositions \ref{theory proposition1}, \ref{theory proposition2} are shown in the Appendix section of the article
	
	By Corollary \ref{characterization operator2} and Proposition \ref{theory proposition2}, it is clear that the problem of solving the derivative of the function $\Psi(X_t)$ with respect to time $t$ is transformed, by two transformations of the $\mathcal{A}$ operator and the $\mathcal{Y}$ operator, into a problem of solving the partial differential equations for the drift and diffusion term functions of the SDEs obeyed by the independent variables in the function $\Psi(X_t)$ with respect to their independent variables which means that solve for the three partial differential matrices $\mathcal{C}^1_{h,x,t}$, $ \mathcal{C}^1_{H,x,t}$ and $\mathcal{C}^2_{H,x,t}$ in the $\mathcal{Y}$ operator.
	
	Furthermore, the transformation from the $\mathcal{A}$ operator to the $\mathcal{Y}$ operator is essentially a transformation from a partial differential operator of the function $\Psi(X_t)$ to a linear operator of the function $\Psi(X_t)$. the $\mathcal{Y}$ operator, as a linear operator of the function $\Psi(X_t)$, would have much better properties, both in mathematics and engineering.

	\section{SYSTEM MODELING FOR A CLASS OF CHILD-MOTHER SYSTEM}
	\label{Model}
	\subsection{A Class Of Child-Mother System Modeling}
	Consider a system full of uncertainties, with deterministic models that current researchers have jointly established through observations and physical laws and uncertainty models that cannot currently be described by mathematical models. Fortunately, even if the researcher cannot fully characterize the uncertain part of the system, the uncertain part of the system is observable by sensors, and thus a large amount of data can be obtained to study and analyze that uncertainty. This kind of system can be described by SDEs as
	\begin{equation}
		dO_t = F(O_t, u_t)dt + G(O_t, u_t)dB^{O}_t,
	\end{equation}
	where $O_t\in\mathcal{O}\subset \mathbb{R}^{n_O}$ is the observation of the whole system and $u_t\in\mathcal{U}\subset \mathbb{R}^{d_m}$ is the control input of the system. $\mathcal{O}$ is the set of all observations that can be taken into account and $\mathcal{U}$ is the control space of control inputs, i.e., the set of allowed controls. $B^{O}_t$ is a $n_O$-dimensional Wiener motion. $F(O_t, u_t):\mathbb{R}^{n_O}\times \mathbb{R}^{d_m}\rightarrow \mathbb{R}^{n_O}$, $G(O_t, u_t):\mathbb{R}^{n_O}\times \mathbb{R}^{d_m}\rightarrow \mathbb{R}^{n_O\times n_O}$ are the drift term and diffusion term of SDEs, respectively. 
	
	In this paper, we focus on a class of systems in which the system can be decomposed into a child-mother system as the following form.
	\begin{equation}
		\label{system}
		O_t = \begin{bmatrix}
			z_t\\
			w_t
		\end{bmatrix},
		\begin{aligned}
			dz_t &= f^{\theta_z}(z_t,u_t)dt+g^{\theta_z}(z_t,u_t)dB^{z}_{t}, \\
			dw_t &= \alpha^{\theta_w}(w_t, z_t)dt +\beta^{\theta_w}(w_t, z_t)dB^{w}_t,\\
		\end{aligned}
	\end{equation} 
	where $z_t\in\mathcal{Z}\subset \mathbb{R}^{d_n}$ is the states of subsystem and $w_t\in\mathcal{W}\subset \mathbb{R}^{s_d}$ is the states of mother system which means $w_t$ will be effected by $z_t$. $\mathcal{Z}$ represents the state space of the subsystem state, i.e., the set consisting of all state values that can be taken and $\mathcal{W}$ represents the state space of the state of the mother system, i.e., the set of all values to which the state values of the mother system can be taken.
	$B^{z}_{t}$ and $B^{w}_{t}$ are the $d_n$-dimensional and $s_d$-dimensional Wiener motion, respectively. $f^{\theta_z}:\mathbb{R}^{d_n}\times\mathbb{R}^{d_m}\rightarrow \mathbb{R}^{d_n}$ and $g^{\theta_z}:\mathbb{R}^{d_n}\times \mathbb{R}^{d_m}\rightarrow \mathbb{R}^{d_n\times d_n}$ are the drift term and diffusion term of SDEs, respectively.  Likewise, $\alpha^{\theta_w}:\mathbb{R}^{s_d}\times\mathbb{R}^{d_n}\rightarrow \mathbb{R}^{s_d}$ and $\beta^{\theta_w}:\mathbb{R}^{s_d}\times\mathbb{R}^{d_n}\rightarrow \mathbb{R}^{s_d\times s_d}$ are the drift term and diffusion term of SDEs, respectively.
	
	The drift term and diffussion term functions $f^{\theta_z}, g^{\theta_z}, \alpha^{\theta_w}$ and $\beta^{\theta_w}$ of SDEs(\ref{system}) can be calibrated by the large amount of data observed and collected by the sensors from the system. The parmeters $\theta_z$, $\theta_w$ are the parameters that the neural network needs to learn.

	According to Lemma \ref{existence and uniqueness}, if SDEs(\ref{system}) has a unique solution, conditions $(i)\sim (iv)$ of Lemma \ref{existence and uniqueness}  must be obeyed. It's trivial that conditions $(iii)$ and $(iv)$ are naturally true for the child-mother system in (\ref{system}). Therefore, the drift term and diffusion term of SDEs in (\ref{system}) need to satisfy:
	\begin{equation}
		(i)
		\begin{aligned}
			&\begin{cases}
				&\abs{f^{\theta_z}(z_i,u_k)-f^{\theta_z}(z_j,u_k)}\\
				&\qquad+\abs{g^{\theta_z}(z_i,u_k)-g^{\theta_z}(z_j,u_k)}\leq L_1\abs{z_i-z_j},\\
				
				&\abs{f^{\theta_z}(z_k, u_i)-f^{\theta_z}(z_k, u_j)}\\
				&\qquad+\abs{g^{\theta_z}(z_k, u_i)-g^{\theta_z}(z_k, u_j)}\leq L_2\abs{u_i-u_j},\\
				
				&\abs{\alpha^{\theta_w}(w_i, z_k)-\alpha^{\theta_w}(w_j, z_k)}\\
				&\qquad+\abs{\beta^{\theta_w}(w_i, z_k)-\beta^{\theta_w}(w_j, z_k)}\leq L_3\abs{w_i-w_j},\\
				
				&\abs{\alpha^{\theta_w}(w_k, z_i)-\alpha^{\theta_w}(w_k, z_j)}\\
				&\qquad+\abs{\beta^{\theta_w}(w_k, z_i)-\beta^{\theta_w}(w_k, z_j)}\leq L_4\abs{z_i-z_j},\\
			\end{cases}\\
			&\qquad w_i,w_j,w_k\in \mathcal{W}, z_i, z_j ,z_k\in \mathcal{Z}, u_i, u_j, u_k\in \mathcal{U}\\ 
			&\qquad L_1, L_2,L_3, L_4\text{ are constants}
		\end{aligned}
	\end{equation}
	\begin{equation}
		(ii)
		\begin{aligned}
			&\begin{cases}
				&\abs{f^{\theta_z}(z_i, u_i)}+\abs{g^{\theta_z}(z_i, u_i)}\leq K_1(1+\abs{z_i}+\abs{u_i}),\\
				&\abs{\alpha^{\theta_w}(w_i, z_i)}+\abs{\beta^{\theta_w}(w_i, z_i)}\leq K_2(1+\abs{w_i}+\abs{z_i}),\\
			\end{cases}\\
			\qquad& w_i\in \mathcal{W}, z_i \in \mathcal{Z}, u_i\in \mathcal{U},\\
			\qquad& K_1, K_2\text{ are constants}
		\end{aligned}
	\end{equation}

	The output value of the neural network can‘t be infinite that the condition $(ii)$ is also naturally true. It's obvious that condition $(i)$ means that $\alpha^{\theta_w},  \beta^{\theta_w}, f^{\theta_z}, g^{\theta_z}$ have to satisfy Lipschitz continuous condition. Therefore, it need to add a gradient penalty in the objective functions as a regularizer which is:
	\begin{equation}
		\begin{cases}
			\mathcal{L}_{L}(\theta_w) &= \kappa_1\int_{0}^{T}(\norm{\partial_w\alpha^{\theta_w}(w_t,z_t)}_2+\norm{\partial_z\alpha^{\theta_w}(w_t,z_t)}_2\\
			&+\norm{\partial_w\beta^{\theta_w}(w_t,z_t)}_2+\norm{\partial_z\beta^{\theta_w}(w_t,z_t)}_2-C_1)^+dt,\\
			\mathcal{L}_{L}(\theta_z) &= \kappa_2\int_{0}^{T}(\norm{\partial_z f^{\theta_z}(z_t,u_t)}_2+\norm{\partial_u f^{\theta_z}(z_t,u_t)}_2\\
			&+\norm{\partial_z g^{\theta_z}(z_t,u_t)}_2+\norm{\partial_u g^{\theta_z}(z_t,u_t)}_2-C_2)^+dt,\\
		\end{cases}
		\label{J_lip}
	\end{equation}
	where $\kappa_1, \kappa_2, C_1, C_2$ are hyper-parameters. $(x)^+$ is the positive part of $x$ which means $(x)^+=\max(0,x)=Relu(x)$.
	
	\subsection{Stochastic Differential Equation Calibration of Child-Mother System}
	For a typical child-mother system, the deterministic part of the system (in SDEs that is, the drift term function) is, for the most part, uniquely and deterministically described in mathematical form by physical laws. However, the stochastic part of the system is very difficult to be well described, at least under the existing theories, and the uncertainty that these stochastic factors (e.g., complex human behaviors, complex high-dimensional stochastic variables in the environment) bring to the system cannot be described by deterministic dynamic equations (e.g., ordinary differential equations, partial differential equations).
	
	Therefore, the modeling of the stochastic part of the system requires a large number of sensors observing the system from reality. And an attempt is made to extract the distribution obeyed by the model of the stochastic part of the system from the large amount of realistic data.
	
	In order to make the system discussed in this paper more general, it is assumed that the parameters in the SDEs in (\ref{system}) are unknown, i.e., the child-mother system is unknown both to the agent in reinforcement learning and to the researchers, and all that is available is the data that the sensors have collected from reality.
	
	According to (\ref{system}), the states of the child-mother system obeys that
	\begin{equation}
		\label{system distribution}
		\begin{aligned}
			z_{t+} &\sim \mathcal{N}(z_t+f^{\theta_z}(z_t,u_t)dt ,g^{\theta_z}(z_t, u_t)g^{\theta_z}(z_t, u_t)^Tdt),\\
			w_{t+}&\sim \mathcal{N}(w_t+\alpha^{\theta_w}(w_t,z_t)dt ,\beta^{\theta_w}(w_t, z_t)\beta^{\theta_w}(w_t,z_t)^Tdt),\\
		\end{aligned}
	\end{equation}
	where $z_{t+}$, $w_{t+}$ are the states of subsystem and mother system respectively at the next moment of time $t$. It is obvious that systems satisfying SDEs have Markovianity in continuous time.
	
	According to Bayesian estimation,
	\begin{equation}
		\begin{aligned}
			P(\theta_{z}\vert\bar{z}_{t+}) = \frac{P(\bar{z}_{t+}\vert\theta_{z})P(\theta_{z})}{P(\bar{z}_{t+})},\\
			P(\theta_{w}\vert\bar{w}_{t+}) = \frac{P(\bar{w}_{t+}\vert\theta_{w})P(\theta_{w})}{P(\bar{w}_{t+})},\\
		\end{aligned}
	\end{equation}
	where the $ \bar{z}_{t+}$, $\bar{w}_{t+}$ are the states of the subsystem and the mother system, respectively, collected by the sensor in the real environment at the next moment of $t$. The intuitive idea is to maximize $P(\theta_{z}\vert\bar{z}_{t+})$ and $P(\theta_{w}\vert\bar{w}_{t})$ in order to make the SDEs established for $z_{t}$ and $w_{t}$ closer to the probability distribution obeyed by the real trajectories in the dataset. This method is also called  maximum a posteriori estimation(MAP) in the field of machine learning. The main idea is to find the optimal parmeters $\theta^*_{z}$ and $\theta^*_{w}$ such that $P(\bar{z}_{t+}\vert\theta_{z})P(\theta_{z})$ and $P(\bar{w}_{t+}\vert\theta_{w})P(\theta_{w})$ maximized.
	\begin{equation}
		\begin{aligned}
			\theta^{*}_{z} &= \argmax\limits_{\theta_{z}} \int_{0}^{T}P(\bar{z}_{t+}\vert\theta_{z})P(\theta_{z})dt,\\
			\theta^{*}_{w} &= \argmax\limits_{\theta_{w}} \int_{0}^{T}P(\bar{w}_{t+}\vert\theta_{w})P(\theta_{w})dt.
		\end{aligned}
	\end{equation}
	Without loss of generality, for $\theta_{z}$ and $\theta_{w}$, first consider the prior to be a uniform distribution, i.e., $P(\theta_{w}) = 1$, $P(\theta_{z}) = 1$. Thus, the likelihood function is set as follows:
	\begin{equation}
		\begin{aligned}
			\mathcal{L}^{Z}_T(\theta_{z}) = &\int_{0}^{T}P(\bar{z}_{t+}\vert\theta_{z})P(\theta_{z})dt,\\
			=&\int_{0}^{T}\frac{e^{I^{Z}_t}}{(2\pi)^{\frac{d_n}{2}}\abs{\det(g^{\theta_{z}}(\bar{z}_t, \bar{u}_t))}}dt,\\
			I^{Z}_t = &{-\frac{1}{2}(\bar{z}_{t+}-E_{\mu,Z}^t)^T(\Sigma^{Z}_t)^{-1}(\bar{z}_{t+}-E_{\mu,Z}^t)},\\
			\Sigma^{Z}_t = &g^{\theta_{z}}(\bar{z}_t, \bar{u}_t) g^{\theta_{z}}(\bar{z}_t, \bar{u}_t)^T,\\
			E_{\mu,Z}^t = &\bar{z}_{t}+f^{\theta_{z}}(\bar{z}_t, \bar{u}_t).
		\end{aligned}
	\end{equation}
	where $\bar{z}_t$ and $\bar{u}_t$ are the subsystem state values and the control input values for the whole system in dataset at time $t$, respectively.
	\begin{equation}
		\begin{aligned}
			\mathcal{L}^{W}_T(\theta_{w}) = &\int_{0}^{T}P(\bar{w}_{t}\vert\theta_{w})P(\theta_{w})dt,\\
			=&\int_{0}^{T}\frac{e^{I^{W}_t}}{(2\pi)^{\frac{s_d}{2}}\abs{\det(\beta^{\theta_w}(\bar{w}_t,\bar{z}_t))}}dt,\\
			I^{W}_t =&{-\frac{1}{2}(\bar{w}_{t+}-E_{\mu,W}^t)^T(\Sigma^{W}_t)^{-1}(\bar{w}_{t+}-E_{\mu,W}^t)},\\
			\Sigma^{W}_t = &\beta^{\theta_w}(\bar{w}_t,\bar{z}_t)\beta^{\theta_w}(\bar{w}_t,\bar{z}_t)^T,\\
			E_{\mu,W}^t = &\bar{w}_t+\alpha^{\theta_w}(\bar{w}_t,\bar{z}_t).
		\end{aligned}
	\end{equation}
	where $\bar{w}_t$ is the mother system state values in dataset at time $t$.
	
	Therefore, the loss function $\mathcal{L}_Z(\theta_{z})$ of the neural network of $f^{\theta_{z}}(\cdot)$ and $g^{\theta_{z}}(\cdot)$ and the loss function $\mathcal{L}_W(\theta_{w})$ of the neural network of $\alpha^{\theta_{w}}(\cdot)$ and $\beta^{\theta_{w}}(\cdot)$ are defined as
	\begin{equation}
		\label{J_Z}
		\mathcal{L}_Z(\theta_{z}) =- \log\mathcal{L}^{Z}_T(\theta_{z})
	\end{equation}
	and
	\begin{equation}
		\label{J_W}
		\mathcal{L}_W(\theta_{w}) =- \log\mathcal{L}^{W}_T(\theta_{w}),
	\end{equation}
	respectively. When the optimization of parameter $\theta_{z}$ and $\theta_{w}$ are completed, a stochastic model of the child-mother system with uncertainty is obtained, 
	\begin{equation}
		\label{trained system}
		\begin{aligned}
			dz_t = f^{\theta_z*}(z_t, u_t)dt +g^{\theta_z*}(z_t, u_t)dB^{z}_t,\\
			dw_t = \alpha^{\theta_w*}(w_t, z_t)dt +\beta^{\theta_w*}(w_t, z_t)dB^{w}_t.
		\end{aligned}
	\end{equation}
	\section{RL DESIGN BASED ON $\mathcal{Y}$ OPERATOR IN CHILD-MOTHER SYSTEM}
	\label{rl design}
	Currently, the prevailing reinforcement learning methods are built based on AC framework. The reinforcement learning framework determined by pairing the AC framework with deep neural networks has the advantages that the traditional optimal control lacks, such as in the computation of the value function, the reinforcement learning can be used to estimate the action value function through Critic network, which avoids a large number of computations brought about by the inverse solving of the value function in the traditional optimal control (e.g., dynamic programming), and the difference is obvious when both states and actions are continuous, i.e., the state space and action space are infinite (In mathematics, the potential of the set of state spaces and action spaces is $\aleph_0$). Moreover, for systems in continuous state space and continuous action space, in the field of traditional optimal control, it can use the Hamilton-Jacobi-Belman equations for finding the optimal policy, but this method has strong constraints on the design of the system and the cost function. When the model of the controlled system is unknowable or not exhaustible, the traditional optimal control method may be difficult to solve this type of problem effectively.
	
	In summary, this chapter proposes a novel reinforcement learning framework based on the $\mathcal{Y}$ operator which mentioned in \ref{Theory} and applied to the child-mother system built in \ref{Model}. Moreover, the focus of this paper is not on how to calibrate a system built from SDEs using a large amount of data, but rather on the problem of solving the optimal control for that system once the calibration is complete. The calibration method proposed in \ref{Model} is a trivial method, and in the subsequent design of the reinforcement learning framework, by default, the child-mother system has been calibrated by the data, as shown in (\ref{trained system}).

	\subsection{Critic Network Design}
	\label{rl design critic}
	For the reinforcement learning based on AC framework, the value function of action at time $t$ has be defined as
	\begin{equation}
		\label{action value function}
		\begin{aligned}
			Q(z_t,u_t,w_t) = & \mathbb{E} \Bigg\{\int_{t}^{t+s}\gamma^{k-t}R(z_k,u_k,w_k)dk\\
			&+ \gamma^{t+s}V^{\theta_v}(z_{t+s},w_{t+s}) \Bigg\} ,
		\end{aligned}
	\end{equation}
	where $\theta_v$ is the network parmeter of $V^{\theta_v}(z_t, w_t)$. The $V^{\theta_v}(z_t, w_t)$ is the value function of states $z_t$ and $w_t$ at time $t$, representing the potential value of the state to future total rewards. $R(\cdot)$ is the reward function which is related to the states of child-system $z_t$ and mother-system $w_t$ as well as the control input $u_t$ and the functional form of $R(\cdot)$ should be in the form of a polynomial. $\gamma$ is the reward decay factor due to the fact that rewards at more future moments have a much lower impact on the current action than the current reward has on the action.
	
	For the Critic network in the reinforcement learning framework, the task of this network is to give an accurate estimate of the state value function $V^{\theta_v}(z_t, w_t)$ in the current state so that the Actor network can use the value function estimate $V^{\theta_v}(z_t, w_t)$ to select the optimal action that maximizes the reward, hence finding a globally optimal solution in the whole control process. According to the theory of reinforcement learning, when the parameter $\theta_v$ is trained, the value of $Q(z_t,u_t,w_t)$ is independent of the choice of $s$. It is natural to think of an action $u_t$ when it is uniquely determined at time $t$ when the state of the child-system is $z_t$ and the state of the mother system is $w_t$. Then the action's value function at that moment should be a deterministic one, i.e., in  (\ref{action value function}), the value function $Q(z_t,u_t,w_t)$ is independent of the choice of $s$. Therefore, the following equation holds.
	\begin{equation}
		\begin{aligned}
			0=&\frac{\partial Q(z_t,u_t,w_t)}{\partial s} \\
			0=&\mathbb{E}\Bigg\{ \gamma^{s}R(z_{t+s},u_{t+s},w_{t+s})+\gamma^{t+s}\ln\gamma V^{\theta_v}(z_{t+s},w_{t+s})\\
			&+\gamma^{t+s}\frac{dV^{\theta_v}(z_{t+s},w_{t+s})}{ds}\Bigg\}\\
		\end{aligned}
	\end{equation}
	
	Applying Proposition \ref{theory proposition1} to the above equation, the following equation holds.
	\begin{equation}
		\label{critic deduce}
		\begin{aligned}
			0=&\mathbb{E}\left\{ R\left(z_{t+s},u_{t+s},w_{t+s}\right)+\gamma^{t}\ln\gamma V^{\theta_v}\left(z_{t+s},w_{t+s}\right)\right\}\\
			&+\gamma^{t}\big[<\mathcal{Y}_ZV^{\theta_v}(z_{t+s},w_{t+s}),p(z_{t+s})>\\
			&+<\mathcal{Y}_W V^{\theta_v}(z_{t+s},w_{t+s}),p(w_{t+s})>\big]\\
		\end{aligned}
	\end{equation}
	
	Without loss of generality, assume for a moment that $t = 0$. Indeed, when the Critic network training converges, no change in $s$ will have an effect on $ Q(z_t,u_t,w_t)$, regardless of the time $t$ at any moment. Therefore, (\ref{critic deduce}) can be rewritten as
	\begin{equation}
		\label{critic loss function old}
		\begin{aligned}
			0&=\mathbb{E}\left\{ R\left(z_{s},u_{s},w_{s}\right)+\ln\gamma V^{\theta_v}\left(z_{s}, w_{s}\right)\right\}\\
			&+<\mathcal{Y}_ZV^{\theta_v}(z_{s},w_{s}),p(z_{s})>+<\mathcal{Y}_WV^{\theta_v}(z_{s},w_{s}),p(w_{s})>\\
		\end{aligned}
	\end{equation}
	The formula is only related to the reward value at the moment $s$ and the estimation of the value function by the Critic network at the moment $s$, and does not require a partial derivation of the value function, which is a nice property. The method can be remarkably effective when the value function network is not fully understood, such as in some application scenarios of inverse reinforcement learning(IRL), data-driven tasks, etc., where the network structure and parameters of the Critic network are not well known and only a large amount of data is available.
	
	According to the theory of $\mathcal{Y}$ operators in Definition \ref{Y operator}, the $\mathcal{Y}$ operator with respect to the stochastic process $\{z_s\}$ and $V^{\theta_v}(z_{s},w_{s})$ can be described by the following equation.
	\begin{equation}
		\begin{aligned}
			<\mathcal{Y}_Z&V^{\theta_v}(z_{s},w_{s}),p(z_{s})>=\\
			&\int_{\Omega} V^{\theta_v}(z_{s},w_{s}) \{((\Sigma^{Z}_s)^{-1}(z_s - E_{\mu,Z}^s))^T f^{\theta^*_z}(z_s,u_s)\\
			&-  \varLambda^T\mathcal{C}^1_{f,z,s} +\frac{1}{2} \varLambda^T \{  [-(\Sigma^{Z}_s)^{-1}\\
			&+(\Sigma^{Z}_s)^{-1}(z_s - E_{\mu,Z}^s)(z_s - E_{\mu,Z}^s)^T((\Sigma^{Z}_s)^{-1})^T]\\
			&\odot (g^{\theta^*_{z}}(z_s, u_s) g^{\theta^*_{z}}(z_s, u_s)^T) +  \mathcal{C}^2_{g,z,s}\}\varLambda \\
			&- \varLambda^T \mathcal{C}^1_{g,z,s}  (\Sigma^{Z}_s)^{-1}(z_s - E_{\mu,Z}^s) ]\}p(z_s)dz_s
		\end{aligned}
	\end{equation}
	where $\mathcal{C}^1_{f,z,s}$, $\mathcal{C}^1_{g,z,s}$ and $\mathcal{C}^2_{g,z,s}$ are the first-order partial differential equation matrix of the drift term function $f^{\theta^*_z}$, the first-order partial differential equation matrix of the diffusion term function $g^{\theta^*_{z}}$ and the second-order partial differential equation matrix of the diffusion term function $g^{\theta^*_{z}}$ of SDEs (\ref{trained system}), respectively. And $\mathcal{C}^1_{f,z,s}$, $\mathcal{C}^1_{g,z,s}$ and $\mathcal{C}^2_{g,z,s}$ are defined as
	\begin{equation}
		\begin{aligned}
			&\mathcal{C}^1_{f,z,s} = 
			\begin{bmatrix}
				\frac{\partial f^{\theta^*_z}_1(z_s,u_s)}{\partial z_{s,1}}\\
				\vdots\\
				\frac{\partial f^{\theta^*_z}_{d_n}(z_s,u_s)}{\partial z_{s,d_n}}
			\end{bmatrix}\\
			&\mathcal{C}^1_{g,z,s} = \\
			&\begin{bmatrix}
				\frac{\partial (g^{\theta^*_{z}}(g^{\theta^*_{z}})^T)_{1,1}}{\partial z_{s,1}}  & \cdots & \frac{\partial (g^{\theta^*_{z}}(g^{\theta^*_{z}})^T)_{1,d_n}}{\partial z_{s,1}}\\
				\frac{\partial (g^{\theta^*_{z}}(g^{\theta^*_{z}})^T)_{2,1}}{\partial z_{s,2}}  &\cdots &\frac{\partial (g^{\theta^*_{z}}(g^{\theta^*_{z}})^T)_{2,d_n}}{\partial z_{s,2}}\\
				\vdots  & \cdots & \vdots \\
				\frac{\partial (g^{\theta^*_{z}}(g^{\theta^*_{z}})^T)_{d_n,1}}{\partial z_{s,d_n}}  &\cdots & \frac{\partial (g^{\theta^*_{z}}(g^{\theta^*_{z}})^T)_{d_n,d_n}}{\partial z_{s,d_n}}
			\end{bmatrix}\\
			&\mathcal{C}^2_{g,z,s} = \\
			&\begin{bmatrix}
				\frac{\partial^2 (g^{\theta^*_{z}}(g^{\theta^*_{z}})^T)_{1,1}}{\partial z^2_{s,1}}  & \cdots & \frac{\partial^2 (g^{\theta^*_{z}}(g^{\theta^*_{z}})^T)_{1,d_n}}{\partial z_{s,1} \partial z_{s,d_n}}\\
				\frac{\partial^2 (g^{\theta^*_{z}}(g^{\theta^*_{z}})^T)_{2,1}}{\partial z_{s,2} \partial z_{s,1}}  &\cdots &\frac{\partial^2 (g^{\theta^*_{z}}(g^{\theta^*_{z}})^T)_{2,d_n}}{\partial z_{s,2} \partial z_{s,d_n}}\\
				\vdots  & \cdots & \vdots \\
				\frac{\partial^2 (g^{\theta^*_{z}}(g^{\theta^*_{z}})^T)_{d_n,1}}{\partial z_{s,d_n} \partial z_{s,1}}  &\cdots & \frac{\partial^2 (g^{\theta^*_{z}}(g^{\theta^*_{z}})^T)_{d_n,d_n}}{\partial z^2_{s,d_n} }
			\end{bmatrix}\\
		\end{aligned}
	\end{equation}
	Moreover, at moment $s$, the random variable $z_s$ obeys the Gaussian distribution and the expectation matrix $E_{\mu,Z}^s$ and covariance matrix $\Sigma^{Z}_s $ are the following form.
	\begin{equation}
		\begin{aligned}
			E_{\mu,Z}^s = &z_{s-}+f^{\theta^*_{z}}(z_{s-}, u_{s-})ds,\\
			\Sigma^{Z}_s = &g^{\theta^*_{z}}(z_{s-}, u_{s-}) g^{\theta^*_{z}}(z_{s-}, u_{s-})^Tds,\\
		\end{aligned}
	\end{equation}
	where $z_{s-}$ and $u_{s-}$ are the child-system states and the control input at the last moment of time $s$. It can also be thought of as the historical information available to the entire child-mother system at moment $s$ (i.e., the system history information collected by the sensors).
	
	Likewise, the $\mathcal{Y}$ operator with respect to the stochastic process $\{w_s\}$ and $V^{\theta_v}(z_{s},w_{s})$ can be described by the following equation.
	\begin{equation}
		\begin{aligned}
			<\mathcal{Y}_W&V^{\theta_v}(z_{s},w_{s}),p(w_{s})> =\\
			&\int_{\Omega} V^{\theta_v}(z_{s},w_{s}) \{((\Sigma^{W}_s)^{-1}(w_s - E_{\mu,W}^s))^T \alpha^{\theta^*_w}(w_s,z_s)\\
			&-  \varLambda^T\mathcal{C}^1_{\alpha,w,s} +\frac{1}{2} \varLambda^T \{  [-(\Sigma^{W}_s)^{-1}\\
			&+(\Sigma^{W}_s)^{-1}(w_s - E_{\mu,W}^s)(w_s - E_{\mu,W}^s)^T((\Sigma^{W}_s)^{-1})^T]\\
			&\odot (\beta^{\theta^*_{w}}(w_s, z_s) \beta^{\theta^*_{w}}(w_s, z_s)^T) +  \mathcal{C}^2_{\beta,w,s}\}\varLambda \\
			&- \varLambda^T \mathcal{C}^1_{\beta,w,s}  (\Sigma^{W}_s)^{-1}(w_s - E_{\mu,W}^s) ]\}p(w_s)dw_s
		\end{aligned}
	\end{equation}
	where $\mathcal{C}^1_{\alpha,w,s}$, $\mathcal{C}^1_{\beta,w,s}$ and $\mathcal{C}^2_{\beta,w,s}$ are the first-order partial differential equation matrix of the drift term function $\alpha^{\theta^*_w}$, the first-order partial differential equation matrix of the diffusion term function $\beta^{\theta^*_{w}}$ and the second-order partial differential equation matrix of the diffusion term function $\beta^{\theta^*_{w}}$ of SDEs (\ref{trained system}), respectively. And $\mathcal{C}^1_{\alpha,w,s}$, $\mathcal{C}^1_{\beta,w,s}$ and $\mathcal{C}^2_{\beta,w,s}$ are defined as
	\begin{equation}
		\begin{aligned}
			&\mathcal{C}^1_{\alpha,w,s} = 
			\begin{bmatrix}
				\frac{\partial \alpha_1(w_s,z_s)}{\partial w_{s,1}}\\
				\vdots\\
				\frac{\partial \alpha_{d_n}(w_s,z_s)}{\partial w_{s,d_n}}
			\end{bmatrix}\\
			&\mathcal{C}^1_{\beta,w,s} = \\
			&\begin{bmatrix}
				\frac{\partial (\beta^{\theta^*_{w}}(\beta^{\theta^*_{w}})^T)_{1,1}}{\partial w_{s,1}}  & \cdots & \frac{\partial (\beta^{\theta^*_{w}}(\beta^{\theta^*_{w}})^T)_{1,d_n}}{\partial w_{s,1}}\\
				\frac{\partial (\beta^{\theta^*_{w}}(\beta^{\theta^*_{w}})^T)_{2,1}}{\partial w_{s,2}}  &\cdots &\frac{\partial (\beta^{\theta^*_{w}}(\beta^{\theta^*_{w}})^T)_{2,d_n}}{\partial w_{s,2}}\\
				\vdots  & \cdots & \vdots \\
				\frac{\partial (\beta^{\theta^*_{w}}(\beta^{\theta^*_{w}})^T)_{d_n,1}}{\partial w_{s,d_n}}  &\cdots & \frac{\partial (\beta^{\theta^*_{w}}(\beta^{\theta^*_{w}})^T)_{d_n,d_n}}{\partial w_{s,d_n}}
			\end{bmatrix}\\
			&\mathcal{C}^2_{\beta,w,s} = \\
			&\begin{bmatrix}
				\frac{\partial^2 (\beta^{\theta^*_{w}}(\beta^{\theta^*_{w}})^T)_{1,1}}{\partial w^2_{s,1}}  & \cdots & \frac{\partial^2 (\beta^{\theta^*_{w}}(\beta^{\theta^*_{w}})^T)_{1,d_n}}{\partial w_{s,1} \partial w_{s,d_n}}\\
				\frac{\partial^2 (\beta^{\theta^*_{w}}(\beta^{\theta^*_{w}})^T)_{2,1}}{\partial w_{s,2} \partial w_{s,1}}  &\cdots &\frac{\partial^2 (\beta^{\theta^*_{w}}(\beta^{\theta^*_{w}})^T)_{2,d_n}}{\partial w_{s,2} \partial w_{s,d_n}}\\
				\vdots  & \cdots & \vdots \\
				\frac{\partial^2 (\beta^{\theta^*_{w}}(\beta^{\theta^*_{w}})^T)_{d_n,1}}{\partial w_{s,d_n} \partial w_{s,1}}  &\cdots & \frac{\partial^2 (\beta^{\theta^*_{w}}(\beta^{\theta^*_{w}})^T)_{d_n,d_n}}{\partial w^2_{s,d_n} }
			\end{bmatrix}\\
		\end{aligned}
	\end{equation}
	Moreover, the $E_{\mu,W}^s$ and $\Sigma^{W}_s $ are defined as
	\begin{equation}
		\begin{aligned}
			E_{\mu,W}^s = &w_{s-}+\alpha^{\theta^*_{w}}(w_{s-}, z_{s-})ds,\\
			\Sigma^{W}_s = &\beta^{\theta^*_{w}}(w_{s-}, z_{s-}) \beta^{\theta^*_{w}}(w_{s-}, z_{s-})^Tds,\\
		\end{aligned}
	\end{equation}
	where $w_{s-}$ and $z_{s-}$ are the states of mother-system and the states of child-system states at the last moment of time $s$. 
	
	The goal of training for the Critic network is essentially to find the optimal parameter $\theta^*_{v}$ such that (\ref{critic loss function old}) holds at any $s$ moments.
	Therefore, the loss function of the Critic network can be set intuitively
	\begin{equation}
		\begin{aligned}
			\mathcal{L}_C(\theta_v) &=\int_{0}^{T} \Big(R(z_{s},u_{s},w_{s})+\ln\gamma V^{\theta_v}(z_{s},w_{s})\\
			&+\mathcal{Y}_Z V^{\theta_v}(z_{s},w_{s})+\mathcal{Y}_W V^{\theta_v}(z_{s},w_{s})\Big)^2 ds.\\
		\end{aligned}
	\end{equation}
	More generally, at any given time $t$, all selections of $s$ should make (\ref{action value function}) a fixed value, and the more generalized form of the loss function should be reformulated as
	\begin{equation}
		\begin{aligned}
			&\mathcal{L}_C(\theta_v) =\int_{0}^{T}\int_{0}^{T-t} \Big(R(z_{t+s},u_{t+s},w_{t+s})\\
			&+\gamma^t [ \ln\gamma V^{\theta_v}(z_{t+s},w_{t+s})+\mathcal{Y}_ZV^{\theta_v}(z_{t+s},w_{t+s})\\
			&+\mathcal{Y}_WV^{\theta_v}(z_{t+s},w_{t+s})]\Big)^2 dsdt.\\
		\end{aligned}
	\end{equation}
	
	When the Critic network is well trained by a large amount of data collected offline or online, the network can give, at a given moment, the value function corresponding to the current moment. Denote $\theta^*_v$ as the network parameter when the Critic network is well trained.
	
	\subsection{Actor Network Design}
	For reinforcement learning based on the AC architecture, the goal of the Actor network is to output the optimal action in the current state after the value function of the current state is given by the Critic network, which maximizes the final reward.
	The action ouputted by Actor network can be represented as
	\begin{equation}
		u_t = \mu^{\theta_u}(z_t, w_t)+\sigma^{\theta_u}(z_t, w_t)B^{u}_{\Delta},
	\end{equation}
	where $ \mu^{\theta_u}:\mathbb{R}^{d_n} \times \mathbb{R}^{s_d}\rightarrow \mathbb{R}^{d_m}$, $\sigma^{\theta_u}:\mathbb{R}^{d_n} \times \mathbb{R}^{s_d}\rightarrow \mathbb{R}^{d_m\times d_m}$ are the learning-based function of controller which can be determined by neural network with the parameter $\theta_u$.  $B^{u}_{\Delta}$ is a $d_m$-dimensional random vector. The random variables in random vector $B^{u}_{\Delta}$ are independently and identically distributed random variables, all of which follow the Gaussian distribution with expectation $0$ and variance $\Delta T$(time interval of control inputs). Therefore, at time $t$, the action ouputted by Actor network obeys the following distribution:
	\begin{equation}
		\label{u_distribution}
		u_t \sim \mathcal{N}(\mu^{\theta_u}(z_t, w_t), \Delta T\cdot\sigma^{\theta_u}(z_t, w_t)\times \sigma^{\theta_u}(z_t, w_t)^T),
	\end{equation}
	where $\mu^{\theta_u}(z_t, w_t)$ is the expectation vector and $\Delta T\cdot\sigma^{\theta_u}(z_t, w_t)\times \sigma^{\theta_u}(z_t, w_t)^T$ is the covariance matrix.

	The key to improve the policy of Actor is to find the optimal parmeter $\theta^*_u$ of Actor network that maximize the reward in reinforcement learning. It means that find
	\begin{equation}
		\label{actor design old}
		\begin{aligned}
			\theta^*_u = \argmin\limits_{\theta_u} \int_{0}^{T}&-\mathbb{E}\{\log\pi^{\theta_u}(u_t\vert z_t, w_t)\}\hat{A}_tdt,\\
			\hat{A}_t = &Q(z_t,u_t,w_t)-V^{\theta^*_v}(z_{t},w_{t})
		\end{aligned}
	\end{equation}
	where $\pi^{\theta_u}(u_t\vert z_t, w_t)$ is the conditional distribution of action $u_t$ when the states of child-system $z_t$ and states of mother-system $w_t$ is well known at time $t$. $\hat{A}_t$ is the advantage function, which indicates whether the current action is a favorable action for the current value function, when $\hat{A}_t$ is greater than $0$, it means that the current action is an advantageous action, on the contrary, the action is a worse action.

	According to the Actor design in TRPO\cite{schulman2015trust} and PPO\cite{schulman2017proximal}, (\ref{actor design old}) can be reformulated as
	\begin{equation}
		\begin{aligned}
			\theta^*_u = \argmin\limits_{\theta_u} \int_{0}^{T}&-\mathbb{E}\bigg[\frac{\pi^{\theta_u}(u_t\vert z_t, w_t)}{\pi^{\theta_u}_{old}(u_t\vert z_t, w_t)}\bigg]\hat{A}_tdt,\\
		\end{aligned}
	\end{equation}
	where $\pi^{\theta_u}_{old}$ represents the conditional probability of the action under the old strategy.
	
	Most current studies comparing the advantages of algorithms use the PPO algorithm as a baseline for comparison. The subsequent experimental results in this paper are also compared with PPO, therefore, the design of Actor takes the same design method as the PPO algorithm. 
	Thus, define $\mathcal{L}_a(\theta_u)$ as the loss function of the Actor as follows
	\begin{equation}
		\label{J_actor}
		\begin{aligned}
			\mathcal{L}_A(\theta_u)=\int_{0}^{T}-\min(\mathcal{R}^{\theta_u}_t&\hat{A}_t, \text{clip}(\mathcal{R}^{\theta_u}_t,1-\epsilon,1+\epsilon)\hat{A}_t)dt\\
			\mathcal{R}^{\theta_u} _t=& \frac{\pi^{\theta_u}(u_t\vert z_t, w_t)}{\pi^{\theta_u}_{old}(u_t\vert z_t, w_t)}
		\end{aligned}
	\end{equation}
	where "clip" is a function that restricts $\mathcal{R}^{\theta_u}_t$ to the middle of $1-\epsilon$ and $1+\epsilon$. $\epsilon$ is a hyperparameterization.
	
	In practice, during reinforcement learning training, the Actor network and the Critic network are not trained first after the Critic network is trained and then the Actor network is trained. Instead, after the agent interacts with the environment for a fixed number of episodes, the data from these episodes are used, either synchronously or asynchronously, to train the Actor network and Critic network.

	\section{ILLUSTRATIVE EXAMPLES}
	\label{simulation}
	Comparing the YORL method with the TSRL method, both YORL and TSRL use a design consistent with the PPO method in the design of the Actor network. However, in the design of the Critic network, YORL uses the loss function update network described in \ref{rl design critic}, and TSRL uses the loss function update network described in (\ref{loss tradition}). In fact TSRL's AC network design is all consistent with PPO, simply by applying the PPO algorithm to a system modeled by SDEs.
	\subsection{Linear Numerical Examples}
	
	Considering a linear child-mother system
	\begin{equation}
		\begin{aligned}
			\begin{bmatrix}
				dz^{(1)}_t\\
				dz^{(2)}_t
			\end{bmatrix}
			&=\begin{bmatrix}
				z^{(2)}_t\\
				u_t
			\end{bmatrix}dt\\
			&+\begin{bmatrix}
				0.1z^{(2)}_t+0.1 & 0 \\
				0 & 0.1z^{(2)}_t+0.3
			\end{bmatrix}
			\begin{bmatrix}
				dB^{z,1}_t\\
				dB^{z,2}_t
			\end{bmatrix}\\
			\begin{bmatrix}
				dw^{(1)}_t\\
				dw^{(2)}_t
			\end{bmatrix}
			&=\begin{bmatrix}
				w^{(2)}_t\\
				z^{(2)}_t-w^{(2)}_t
			\end{bmatrix}dt\\
			&+\begin{bmatrix}
				0.2w^{(2)}_t+0.3 & 0 \\
				0 & 0.1w^{(2)}_t +0.2
			\end{bmatrix}
			\begin{bmatrix}
				dB^{w,1}_t\\
				dB^{w,2}_t
			\end{bmatrix}\\
		\end{aligned}
	\end{equation}
	The reward function can be determined by
	\begin{equation}
		\begin{aligned}
			R(z_t,w_t,u_t) = c_1 e^{-c_2[\abs{z^{(1)}_t-w^{(1)}_t}-c_3]^2}+c_4u^2_t
		\end{aligned}
	\end{equation}
	where $c_1, c_2, c_4$ are the constant and $c_3$ is the magnitude of the distance expected to be maintained by $z^{(1)}_t$ and $w^{(1)}_t$.
	The goal of optimal control is to find the optimal control policy that maximizes the total reward while satisfying the constraint equations.
	\begin{equation}
		\begin{aligned}
			&\max \int_{0}^{T}R(z_t,w_t,u_t)dt\\
			s.t. \quad&u_t\in [-2,2], z^{(1)}_t, w^{(1)}_t\in [0,300],\\
			& z^{(2)}_t, w^{(2)}_t \in [0,10], w^{(1)}_t-z^{(1)}_t\geq 0,\\
			& z_0 = \begin{bmatrix}
				0\\2
			\end{bmatrix}, w_0 = \begin{bmatrix}
			8\\4
			\end{bmatrix}
		\end{aligned}
	\end{equation}
	The specific design details of the reinforcement learning framework are shown in Table  \ref{Simulation1}. 
	\begin{table*}[htbp]
		\caption{SPECIFICATION OF THE RL DESIGN IN LINEAR NUMERICAL EXAMPLE}
		\label{Simulation1}
		\setlength{\tabcolsep}{9.8mm}
		\renewcommand\arraystretch{1.2}
		\begin{tabular}{|c|c|c|c|}
			\hline
			\multicolumn{4}{|c|}{RL Design}\\
			\hline
			Actor Network& $1\times\text{hidden dimension}\times1$&$c_1$& $5$\\
			\hline
			Critic Network& $1\times\text{hidden dimension}\times\text{hidden dimension}\times1$&$c_2$& $0.1$\\
			\hline
			$\gamma$& $0.9999$&$c_3$ & $10$\\
			\hline
			$\epsilon$& $0.2$&$c_4$ & $-0.2$\\
			\hline
		\end{tabular}
	\end{table*}
	
	A training comparison between the reinforcement learning method using the $\mathcal{Y}$ operator-based design of the loss function of the Critic network(YORL) and the TSRL method on this linear numerical example is shown in Fig.\ref{YORL v.s PPO} .
	\begin{figure}[htbp]
		\centering
		\includegraphics[width=\columnwidth]{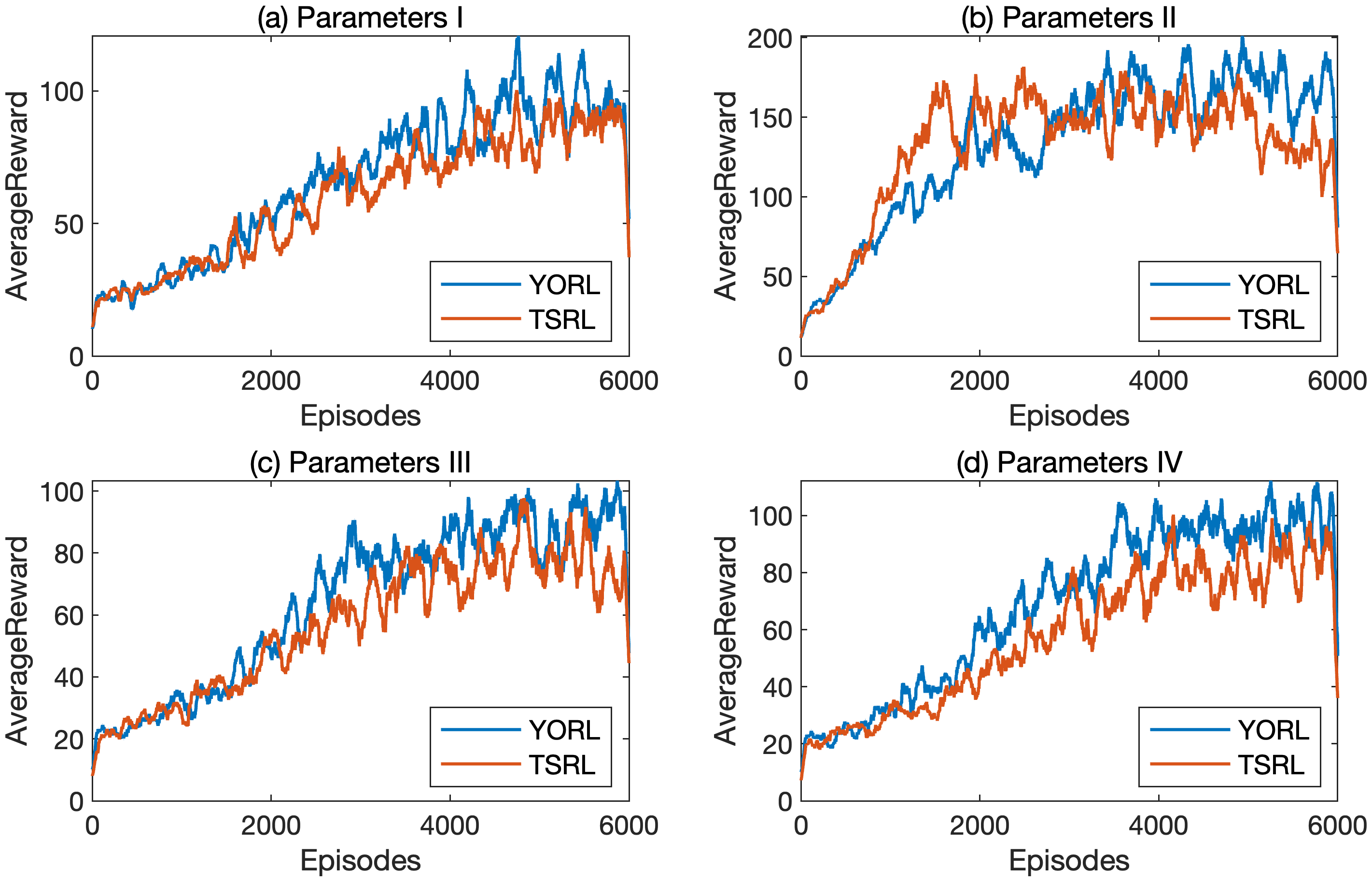}
		\caption{In subfigure (a), the hidden layer dimension of both the Actor and Critic networks for both TSRL and YORL is $32$ and the activation function used is the Sigmoid function. In subfigure (b), the hidden layer dimension is $128$ and the activation function used is Sigmoid. In subfigure (c), the hidden layer dimension is $32$ and the activation function used is Relu. In subfigure (d), the hidden layer dimension is $32$ and the activation function used is tanh.}
		\label{YORL v.s PPO}
	\end{figure}
	
	It is clear that the training performance of the YORL method using the $\mathcal{Y}$ operator is outperformed by the TSRL method in this linear child-mother system and the YORL method is able to achieve higher rewards than the TSRL method  at the time of training convergence.
	This is due to the fact that YORL incorporates the stochasticity of the child-mother system modeled by SDEs into the design of the Critic network, and therefore is outperforming the TSRL, a method that does not take into account the stochasticity in the model, in terms of the reward function.
	It nicely addresses the deficiencies of the current study as articulated in \textbf{Problem I}.
	
	As can be seen in Fig.\ref{YORL v.s PPO}, when the Relu function is used for the activation function of the network, there is a certain degree of reduction in the reward function for both the YORL and TSRL methods, in consistent with what was described in \textbf{Problem II}. However, the value of reward function of the YORL method is still higher than that of the conventional TSRL method because the $\mathcal{Y}$ operator avoids the problem of calculating the partial differentiation of the state value function.
	
	The experimental results reveal that the network with Relu activation function has faster convergence characteristics when the hidden layer dimensions are the same. While the network with Sigmoid activation function is slower in convergence, the peak size of the reward after convergence performs better compared to other networks. Compared to Relu and Sigmoid, the network with tanh activation function has a more moderate performance. In addition, the TSRL method  converges faster than YORL when the hidden layer dimension increases, but the average reward after convergence is inferior to the YORL method.
	\subsection{Nonlinear Numerical Examples}
	Considering a nonlinear child-mother system
	\begin{equation}
		\begin{aligned}
			\begin{bmatrix}
				dz^{(1)}_t\\
				dz^{(2)}_t
			\end{bmatrix}
			&=\begin{bmatrix}
				z^{(2)}_t\\
				u_t- (0.1z^{(2)}_t)^2- 0.5\sin(z^{(1)}_t)
			\end{bmatrix}dt\\
			&+\begin{bmatrix}
				0.1z^{(2)}_t+0.1 & 0 \\
				0 & 0.1z^{(2)}_t+0.3\\
			\end{bmatrix}
			\begin{bmatrix}
				dB^{z,1}_t\\
				dB^{z,2}_t\\
			\end{bmatrix}\\
			\begin{bmatrix}
				dw^{(1)}_t\\
				dw^{(2)}_t
			\end{bmatrix}
			&=\begin{bmatrix}
				w^{(2)}_t\\
				z^{(2)}_t-w^{(2)}_t-(0.1w^{(2)}_t)^2- 0.5\sin(w^{(1)}_t)
			\end{bmatrix}dt\\
			&+\begin{bmatrix}
				0.2w^{(2)}_t+0.3 &0 \\
				0 & 0.1w^{(2)}_t +0.2
			\end{bmatrix}
			\begin{bmatrix}
				dB^{w,1}_t\\
				dB^{w,2}_t
			\end{bmatrix}\\
		\end{aligned}
	\end{equation}
	
	The reward function can be determined by
	\begin{equation}
		\begin{aligned}
			R(z_t,w_t,u_t) = c_1 e^{-c_2[\abs{z^{(1)}_t-w^{(1)}_t}-c_3]^2}+c_4u^2_t
		\end{aligned}
	\end{equation}
	where $c_1, c_2, c_4$ are the constant and $c_3$ is the magnitude of the distance expected to be maintained by $z^{(1)}_t$ and $w^{(1)}_t$.
	The goal of optimal control is to find the optimal control policy that maximizes the total reward while satisfying the constraint equations.
	\begin{equation}
		\begin{aligned}
			&\max \int_{0}^{T}R(z_t,w_t,u_t)dt\\
			s.t. \quad&u_t\in [-2,2], z^{(1)}_t, w^{(1)}_t\in [0,300],\\
			& z^{(2)}_t, w^{(2)}_t \in [0,10],   w^{(1)}_t-z^{(1)}_t\geq 0,\\
			&z_0 = \begin{bmatrix}
				0\\2
			\end{bmatrix}, w_0 = \begin{bmatrix}
				8\\4
			\end{bmatrix}
		\end{aligned}
	\end{equation}
	The specific design details of the reinforcement learning framework are shown in Table  \ref{Simulation1} as well.
	
	\begin{figure}[htbp]
		\centering
		\includegraphics[width=\columnwidth]{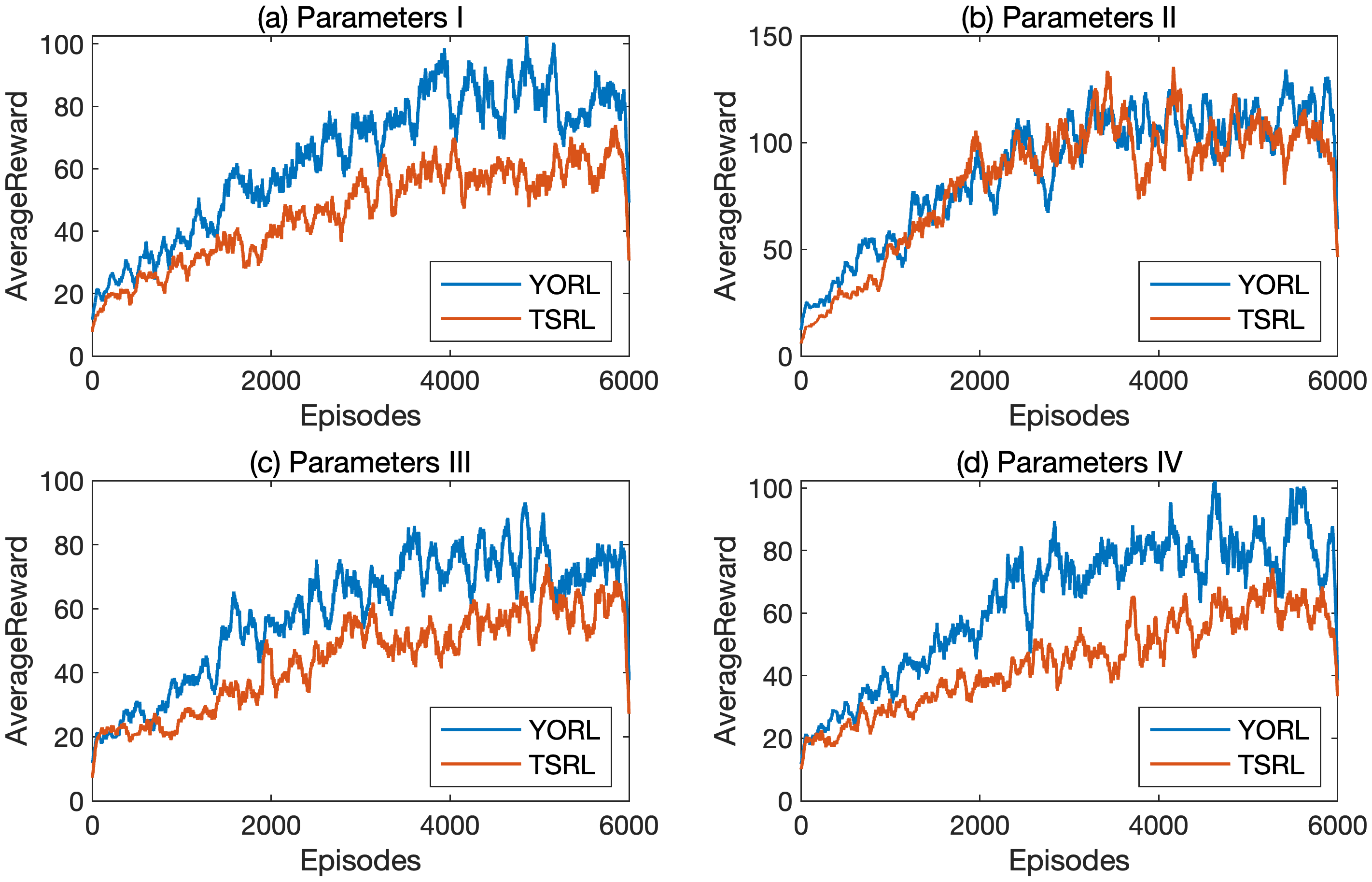}
		\caption{In subfigure (a), the hidden layer dimension of both the Actor and Critic networks for both TSRL and YORL is $32$ and the activation function used is the Sigmoid function. In subfigure (b), the hidden layer dimension is $128$ and the activation function used is Sigmoid. In subfigure (c), the hidden layer dimension is $32$ and the activation function used is Relu. In subfigure (d), the hidden layer dimension is $32$ and the activation function used is tanh.}
		\label{YORL v.s PPO Nonlinear}
	\end{figure}
	
	The simulation results for this nonlinear child-mother system are shown in Fig.\ref{YORL v.s PPO Nonlinear} . Similar to the linear child-mother system, the dimension of the hidden layer used in subfigure (a) is $32$, and the activation function is a Sigmoid function. It can be found that the average reward of YORL is totally higher than the existing TSRL method throughout the training process. When the hidden layer dimension is increased to $128$ and the activation function is unchanged, as shown in subfigure (b), both the TSRL method  and the YORL are improved, which is a positive result of the increase of neurons in the neural network. When the training converges, it can be noticed that the average reward of the YORL method is slightly better than the TSRL method.
	When the dimension of the hidden layer is unchanged and the activation function is replaced with the Relu function, as shown in subfigure (c), the advantage of YORL is still obvious.
	Finally, in subfigure (d), the hidden layer dimension is still $32$ and the activation function is replaced with the tanh function. YORL continues to perform remarkably well.
	
	From the two numerical examples, it can be seen that the YORL method shows better control performance than the TSRL method, both in systems modeled by linear SDEs and in systems modeled by nonlinear SDEs.The YORL method solves \textbf{Problem I} well and gives a feasible approach in some special scenarios described in \textbf{Problem II}. The YORL method can provide an illuminating proposal for researchers when the problem they are studying is in the special scenarios described in \textbf{Problem II}.
	\section{CONCLUSIONS}
	\label{conclusion}
	It can be found that YORL method outperforms TSRL method in both linear child-mother systems and nonlinear child-mother systems, which is a well-suited answer to the \textbf{Problem I} posed in the previous section. This demonstrates the necessity of taking the stochasticity in the system into account in the design of the loss function when designing Critic networks.
	
	When the states in the child-mother system are all considered as stochastic processes represented by SDEs, the state-value function is essentially a functional of the stochastic process. The characteristic operator of the It\^{o} diffusion process $\mathcal{A}$ transforms the problem of solving the derivative of this functional with respect to time into a problem of solving the partial derivative of this functional with respect to its independent variables. The $\mathcal{Y}$ operator proposed in this paper can transform the problem of solving the derivative of the functional with respect to time into the problem of solving the partial derivatives of the drift term function and the diffusion term function of SDEs of the stochastic process with respect to its independent variables.
	
	For solving the transformed problem of the $\mathcal{A}$ operator, it is required that the functional of the stochastic process is known and there is a strict requirement on the continuity of the functional (at least the second-order partial derivatives are required to be continuous). However, in many scenarios, the functional may not be known, and only the data for the inputs and outputs of that functional are available. In designing a Critic network for reinforcement learning, some activation functions such as Relu may not satisfy the continuity requirement, suggesting that the choice of the activation function may result in the functional of the value function that does not satisfy the conditions for the application of the $\mathcal{A}$ operator, which is a significant limitation for the design of the network.
	The $\mathcal{Y}$ operator proposed in this paper transforms this problem into a problem of solving the partial derivatives of the drift term function and the diffusion term function of SDEs obeyed by the random variable itself with respect to its independent variables. The continuity requirement of the pairwise value function is transformed into the continuity requirement of the drift term function and the diffusion term function of SDEs obeyed by the random variable itself.
	
	The proposed $\mathcal{Y}$ operator allows more selectivity in the design of the Critic network when the SDEs obeyed by the system state satisfy the continuity condition. The activation functions of network such as the Relu function can be chosen. Moreover, the $\mathcal{Y}$ operator can be used as a reference tool in areas such as inverse reinforcement learning, offline reinforcement learning, to provide researchers with inspiring proposals. Therefore, the proposed $\mathcal{Y}$ operator effectively addresses the shortcomings articulated in \textbf{Problem II}.
	
	\section{APPENDIX}
	\subsection{The Proof of Proposition \ref{theory proposition1}}
	\begin{proof}
		By denoting the dimension of $X^{(1)}_{t}$ as $d_{x^1}$, the dimension of $X^{(2)}_{t}$ as $d_{x^2}$, and so on, and the dimension of $X^{(m)}_{t}$ as $d_{x^m}$, there are, in essence, $\sum_{i=1}^{m}d_{x^i}$ independent variables in the function $\Psi(X^{(1)}_{t},X^{(2)}_{t}\dots X^{(m)}_{t})$. Denoted $\sum_{i=1}^{m}d_{x^i}$ as $d_M$
		
		According to the Taylor expansion formula for multivariate functions, (\ref{theory proposition1 eq}) in Proposition \ref{theory proposition1} can be written in the following form:
		\begin{equation}
			\label{proposition1 proof}
			\begin{aligned}
				&\mathbb{E}\left\{\frac{d \Psi(X^{(1)}_{t},X^{(2)}_{t}\dots X^{(m)}_{t})}{dt}\right\} = \\
				&\mathbb{E}\Bigg\{ \sum_{i=1}^{m}\sum_{j=1}^{d_{x^i}}\frac{d \Psi(X^{(1)}_{t},X^{(2)}_{t}\dots X^{(m)}_{t})}{d X^{(i)}_{t,j}}h^{(i)}_{j}(X^{(1)}_{t},X^{(2)}_{t}\dots X^{(m)}_{t})\\
				& + \frac{1}{2}\sum_{i=1}^{m}\sum_{j=1}^{d_{x^i}} \sum_{k=1}^{m}\sum_{l=1}^{d_{x^k}} \frac{d \Psi(X^{(1)}_{t},X^{(2)}_{t}\dots X^{(m)}_{t})}{d X^{(i)}_{t,j} d X^{(k)}_{t,l}}\frac{d X^{(i)}_{t,j}}{dt}\frac{d X^{(k)}_{t,l}}{dt}  \Bigg\}
			\end{aligned}
		\end{equation}
		For equation $\mathbb{E}\left\{ \frac{d X^{(i)}_{t,j}}{dt}\frac{d X^{(k)}_{t,l}}{dt} \right\}$,
		\begin{itemize}
			\item Case1($i=k$): 
			\begin{equation}
				\begin{aligned}
					&\mathbb{E}\left\{ \frac{d X^{(i)}_{t,j}}{dt}\frac{d X^{(k)}_{t,l}}{dt} \right\} =\\
					&\left(H^{(i)}\left(X^{(1)}_{t}\dots X^{(m)}_{t}\right)H^{(i)}\left(X^{(1)}_{t}\dots X^{(m)}_{t}\right)^T\right)_{j,l}
				\end{aligned}
			\end{equation}
			\item Case1($i\neq k$): 
			Since $dW^{(1)}_{t}$, $dW^{(2)}_{t}$ $\dots$ $dW^{(m)}_{t}$ are independent of each other, it is clear that the following equation holds
			\begin{equation}
				\begin{aligned}
					&\mathbb{E}\left\{ \frac{d X^{(i)}_{t,j}}{dt}\frac{d X^{(k)}_{t,l}}{dt} \right\} = 0
				\end{aligned}
			\end{equation}
		\end{itemize}
		Therefore, the (\ref{proposition1 proof}) can be reformulated as
		\begin{equation}
			\begin{aligned}
				&\mathbb{E}\left\{\frac{d \Psi(X^{(1)}_{t},X^{(2)}_{t}\dots X^{(m)}_{t})}{dt}\right\} = \\
				&\mathbb{E}\Bigg\{ \sum_{i=1}^{m}\sum_{j=1}^{d_{x^i}}\frac{d \Psi(X^{(1)}_{t},X^{(2)}_{t}\dots X^{(m)}_{t})}{d X^{(i)}_{t,j}}h^{(i)}_{j}(X^{(1)}_{t},X^{(2)}_{t}\dots X^{(m)}_{t})\\
				& + \frac{1}{2}\sum_{i=1}^{m}\sum_{j=1}^{d_{x^i}}\sum_{l=1}^{d_{x^i}} \frac{d \Psi(X^{(1)}_{t},X^{(2)}_{t}\dots X^{(m)}_{t})}{d X^{(i)}_{t,j} d X^{(i)}_{t,l}}\times \\
				&\left(H^{(i)}\left(X^{(1)}_{t}\dots X^{(m)}_{t}\right)H^{(i)}\left(X^{(1)}_{t}\dots X^{(m)}_{t}\right)^T\right)_{j,l}  \Bigg\}
			\end{aligned}
		\end{equation}
		As we known, 
		\begin{equation}
			\begin{aligned}
				<\mathcal{A}&_{X^{(i)}}\Psi, p(X^{(i)}_{t})> = \\
				&\mathbb{E}\Bigg\{ \sum_{j=1}^{d_{x^i}}\frac{d \Psi(X^{(1)}_{t},X^{(2)}_{t}\dots X^{(m)}_{t})}{d X^{(i)}_{t,j}}h^{(i)}_{j}(X^{(1)}_{t},X^{(2)}_{t}\dots X^{(m)}_{t})\\
				& + \frac{1}{2}\sum_{j=1}^{d_{x^i}}\sum_{l=1}^{d_{x^i}} \frac{d \Psi(X^{(1)}_{t},X^{(2)}_{t}\dots X^{(m)}_{t})}{d X^{(i)}_{t,j} d X^{(i)}_{t,l}}\times \\
				&\left(H^{(i)}\left(X^{(1)}_{t}\dots X^{(m)}_{t}\right)H^{(i)}\left(X^{(1)}_{t}\dots X^{(m)}_{t}\right)^T\right)_{j,l}  \Bigg\}
			\end{aligned}
		\end{equation}
		Obviously the following equation holds
		\begin{equation}
			\begin{aligned}
				\mathbb{E}\left\{\frac{d \Psi(X^{(1)}_{t},X^{(2)}_{t}\dots X^{(m)}_{t})}{dt}\right\} = \sum_{i=1}^{m}<\mathcal{A}_{X^{(i)}}\Psi, p(X^{(i)}_{t})>
			\end{aligned}
		\end{equation}
		Also due to Theorem \ref{operator equal theorem}, the $\mathcal{A}$ operator is equivalent to the $\mathcal{Y}$ operator, so Proposition \ref{theory proposition1} holds.
	\end{proof}
	\subsection{The Proof of Proposition \ref{theory proposition2}}
	\begin{proof}
		According to Theorem \ref{operator equal theorem}, the (\ref{theory proposition2 eq1}) naturally holds.
		
		And according to Kolmogorov forward equation, 
		\begin{equation}
			\frac{d P(X_t)}{dt} = \mathcal{A}^*_X p(X_t)
		\end{equation}
		Meanwhile, according to Proposition \ref{theory proposition3}, 
		\begin{equation}
			\begin{aligned}
				<\mathcal{A}_X\Psi(X_t), p(X_t)> &= <\mathcal{Y}_X\Psi(X_t), p(X_t)>\\
				&\downarrow\\
				 <\Psi(X_t), \mathcal{A}^*_Xp(X_t)> &= <\Psi(X_t), \mathcal{Y}_Xp(X_t)>
			\end{aligned}
		\end{equation}
		Therefore, the following equation holds
		\begin{equation}
			\mathcal{A}^*_Xp(X_t) = \mathcal{Y}_Xp(X_t)
		\end{equation}
	\end{proof}
	
	\bibliography{USC.bib}
	\bibliographystyle{elsarticle-num}
	
\end{document}